\documentclass{article}
\usepackage{amsmath, amssymb, amscd, amsthm}
\usepackage{graphicx}
\usepackage{hyperref}
\usepackage{makecell}
\usepackage[utf8]{inputenc} 
\usepackage[T1]{fontenc}    
\usepackage{hyperref}       
\usepackage{multirow}
\usepackage{url}            
\usepackage{booktabs}       
\usepackage{amsfonts}       
\usepackage{nicefrac}       
\usepackage{microtype}      
\usepackage[dvipsnames]{xcolor}
\usepackage{subcaption}
\usepackage{bm}
\usepackage[ruled, lined, linesnumbered, commentsnumbered, longend]{algorithm2e}

\usepackage[capitalize,noabbrev]{cleveref}
\theoremstyle{plain}
\newtheorem{theorem}{Theorem}[section]
\newtheorem{proposition}[theorem]{Proposition}
\newtheorem{lemma}[theorem]{Lemma}
\newtheorem{corollary}[theorem]{Corollary}
\theoremstyle{definition}
\newtheorem{definition}[theorem]{Definition}

\theoremstyle{remark}
\newtheorem{remark}[theorem]{Remark}
\usepackage[textsize=tiny]{todonotes}

\DeclareMathOperator*{\argmin}{arg\,min}

\newcommand{\cov}{\text{\normalfont{Cov}}}

\newcommand{\norm}[1]{\left\lVert#1\right\rVert}
\newcommand{\Pac}{\mathcal{P}_{ac}}

\newcommand{\R}{\mathtt{R}_{\lambda}}
\newcommand{\Rn}{\mathtt{R}_{m,n}}

\newcommand{\re}{\mathtt{RE}}
\newcommand{\RE}{\mathtt{RE}_{\lambda}}
\newcommand{\REn}{\mathtt{RE}_{m,n}}

\newcommand{\sR}{\mathtt{R}_{\lambda}^\varepsilon}
\newcommand{\sRn}{\mathtt{R}^{\varepsilon}_{m,n}}

\newcommand{\sre}{\mathtt{sRE}}
\newcommand{\sRE}{\mathtt{sRE}_{\lambda}^\varepsilon}
\newcommand{\sREn}{\mathtt{sRE}_{m,n}^\varepsilon}
\newcommand{\sREnS}{\mathtt{sRE}_{m,n}^{\varepsilon,S}}

\newcommand{\supp}{\text{\normalfont{Supp}}}

\newcommand{\Teps}{T_\varepsilon}
\newcommand{\Tepsn}{T_\varepsilon^n}

\newcommand{\Unif}{\text{Unif}([0,1]^d)}

\title{On Rank Energy Statistics via Optimal Transport: Continuity, Convergence, and Change Point Detection}
\author{Matthew Werenski$^{(\circ)}$ \\ Matthew.Werenski@tufts.edu \and  Shoaib Bin Masud$^{(\dagger)}$ \\ shoaib\_bin.masud@tufts.edu \and James M. Murphy$^{(\ddagger)}$ \\ JM.Murphy@tufts.edu \and  Shuchin Aeron$^{(\dagger)}$ \\ shuchin@ece.tufts.edu%
  \thanks{The first two student authors contributed equally.\newline
\text{$^{(\circ)}$ Department of Computer Science, Tufts University}\newline
  \text{$^{(\dagger)}$Department of Electrical and Computer Engineering, Tufts University}\newline
   \text{$^{(\ddagger)}$ Department of Mathematics, Tufts University}\newline} 
  } 
  
\makeatletter
\def\thanks#1{\protected@xdef\@thanks{\@thanks
        \protect\footnotetext{#1}}}
\makeatother
  
\begin{document}
\maketitle
\begin{abstract}
This paper considers the use of recently proposed optimal transport-based multivariate test statistics, namely rank energy and its variant the soft rank energy derived from entropically regularized optimal transport, for the  unsupervised nonparametric change point detection (CPD) problem. We show that the soft rank energy enjoys both fast rates of statistical convergence and robust continuity properties which lead to strong performance on real datasets. Our theoretical analyses remove the need for resampling and out-of-sample extensions previously required to obtain such rates.  In contrast the rank energy suffers from the curse of dimensionality in statistical estimation and moreover can signal a change point from arbitrarily small perturbations, which leads to a high rate of false alarms in CPD. Additionally, under mild regularity conditions, we quantify the discrepancy between soft rank energy and rank energy in terms of the regularization parameter.  Finally, we show our approach performs favorably in numerical experiments compared to several other optimal transport-based methods as well as maximum mean discrepancy.
\end{abstract}
\section{Introduction}
\label{sec:introduction}

The problem of detecting changes or transitions in a multivariate time series data $(X_t) \subset \mathbb{R}^d$, referred to henceforth as \emph{change point detection (CPD)}, is a central problem in a number of scientific domains \cite{he2006nonparametric, reeves2007review, adams2007bayesian, qi2014novel, cheng2020optimal, cheng2020on, damjanovic2021catboss}. This entails estimating time points $(\tau_i)$ at which the underlying process that generates the data $(X_t)$ changes in a meaningful way. Equivalently, the CPD problem amounts to partitioning the time series data into disjoint segments, with data in each consecutive segment being statistically distinct.  Specifically, we consider an \textit{unsupervised} setting in which no prior examples of change points are made available from which to learn.

Motivated by recent developments in multivariate goodness-of-fit (GoF) tests, based on the the theory of optimal transport (see  \cite{hallin2022measure} for a recent survey) we propose the use of the rank energy \cite{deb2021multivariate} and its numerically and sample efficient variant, the soft rank energy \cite{masud2021multivariate}, for performing CPD. As noted in \cite{deb2021multivariate, deb2021efficiency} there are several advantages in considering rank energy for GoF testing. First, they are distribution free under the null, a property which we numerically observe to be approximately shared by the soft rank energy (See Figure \ref{fig:null_stat}) and that is lacking in other popular multivariate GoF measures such as the maximum mean discrepancy (MMD) \cite{gretton2012kernel}, Wasserstein distances \cite{ramdas2017wasserstein}, and Sinkhorn divergences \cite{feydy2019interpolating}. In the context of CPD, distribution-freeness potentially allows one to select a threshold for detection independent of the underlying distribution. Furthermore, statistical testing based on rank energy is shown to be robust to outliers and has better power for heavy tailed distributions \cite{deb2021multivariate}.  Note while one can consider other OT-rank based GoF measures such as the rank MMD \cite{deb2021multivariate}, Hotteling's-$T^2$ \cite{deb2021efficiency}, and soft rank MMD \cite{masud2021multivariate}, in this paper we focus on rank energy and soft rank energy and leave analogous development for these cases for future investigation.

We make the following fundamental contributions in this paper, keeping in view the practical utility of these tests towards robust CPD.
\begin{enumerate}
    \item \textbf{Wasserstein Continuity Properties of GoF Statistics}: Theorems \ref{thm:re_non_smooth} and \ref{thm:sre_smooth} provide novel analytic insights into rank energy and soft rank energy which explain their behaviors in practice. We show that the soft rank energy is Lipschitz with respect to the Wasserstein-1 metric while the rank energy fails to even be continuous. These properties translate to the smoothness (or lack thereof) of the GoF statistics in the proposed CPD algorithm with respect to small perturbations that are typical of real data. 
    
    \item \textbf{Convergence of Soft Rank Energy to Rank Energy}: 
    In Theorem \ref{thm:sre_to_re_non_asymp}, under appropriate technical conditions, we provide an explicit convergence rate of the soft rank energy to the rank energy in terms of the regularization parameter. We also provide a non-asymptotic result under milder conditions in Theorem \ref{thm:sre_to_re}.  These results relax the conditions required to obtain convergence in existing work on the soft rank energy \cite{masud2021multivariate}. 

    \item \textbf{Realistic and Fast Sample Convergence:}
    In Theorem \ref{thm:sren_to_sre} we establish the fast convergence rate of $n^{-1/2}$ of the plugin estimate of the soft rank energy to its population counterpart. This provides a non-asymptotic guarantee for the soft rank energy using the same set up as is used in the seminal work on rank energy \cite{deb2021multivariate}.  Importantly, this allows our method to avoid using an out-of-sample extension as is required in \cite{masud2021multivariate}, as well as make the most out of limited samples since it does not require a secondary hold out set of samples.
    
    \item \textbf{Applications to CPD}: We numerically investigate and compare the performance of the soft rank energy with other OT-based and popular GoF statistics for CPD \cite{cheng2020optimal,cheng2020on,ahad2022learning,li2015scan}. Our results demonstrate the effectiveness of the soft rank energy, particularly when the data dimension $d$ is large.
\end{enumerate}
\section{Overview of Change Point Detection}  \label{sec:cpd_problem}

Consider a sequence of samples $(X_t) \in \mathbb{R}^d$, which may be finite or infinite, and assume that the sequence can be sequentially partitioned so that $X_1,...,X_{\tau_1} \sim P_1$, $X_{\tau_1+1},...,X_{\tau_2} \sim P_2, X_{\tau_2+1},...,X_{\tau_3} \sim P_3$ and so on. The distributions $P_i$ and the change points $\tau_i$ are not known in advance and must be uncovered.
The goal of a CPD algorithm is to use the observed sequence $(X_t)$ to output a sequence of predicted change points indices $(\hat{\tau}_k)$ such that the sequence $(\hat{\tau}_k)$ is close to the true sequence $(\tau_j)$. In addition to being unsupervised, we focus on the nonparametric setting in which the data generating distributions $(P_j)$ are not assumed to belong to a parametric family of distributions.

One method to estimate change points is the ``sliding window" approach \cite{aminikhanghahi2017survey, cheng2020on, cheng2020optimal} visualized in Figure \ref{fig:sliding_window} and outlined in Algorithm \ref{alg:alg1}. 
For each time
$t\in \{n,n+1,\dots,T-n\}$ let $z_t$ be a GoF statistic computed between the time-adjacent sets $\{X_{t-n+1},...,X_t\}$, and $\{X_{t+1},...,X_{t+n} \}$.  Repeating this for every time $t$ creates a sequence $(z_t)$ from which a sequence $(\hat{\tau}_i)$ of predicted change points can be extracted. For example, one can predict that $\hat{\tau}$ is a change point when $z_{\hat{\tau}}$ takes a large value or is a local maximizer within the sequence $(z_\tau)$. Since the predicted change points $(\hat{\tau}_k)$ are extracted from the sequence $(z_\tau)$ which is determined by a GoF statistic, the choice of GoF statistic will make a substantial difference in the quality of the predicted change points. The purpose of this paper is to argue theoretically and empirically that the soft rank energy (given in Definition \ref{def:rank_energy}) is a strong choice of statistic, due to its favorable theoretical and computational  properties. 
\begin{figure}[htbp!]
    \centering
    \includegraphics[width =\linewidth]{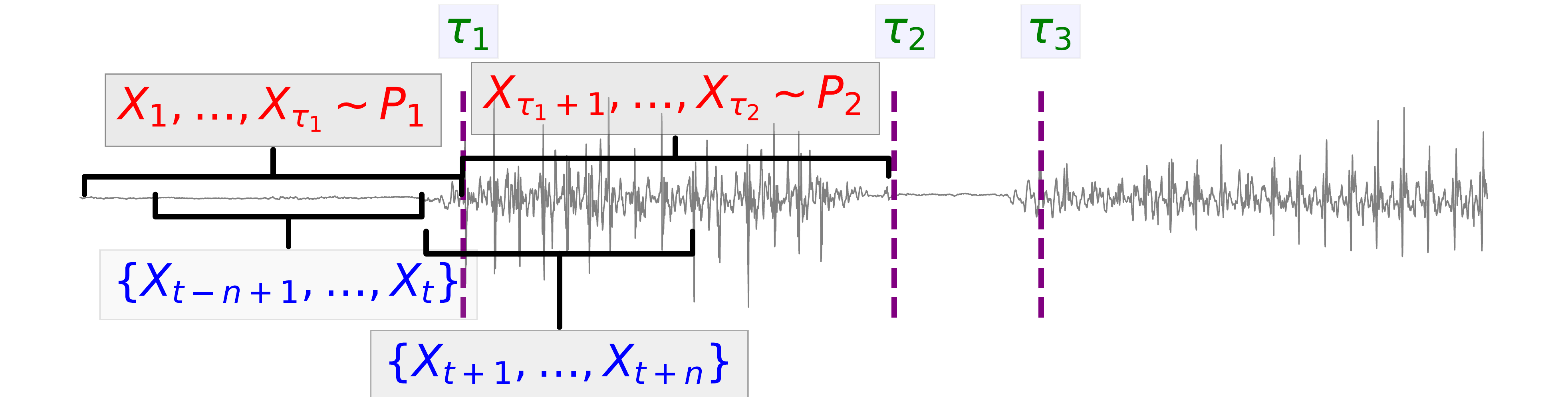}
    \vspace{-5mm}
\caption{$(X_t)_{t=1}^{T}$ is time series data (gray) with several change points (dashed purple lines). The window of $n$ samples between $t$ and $t+n$ is $X_t,...,X_{t+n}$.}
    \label{fig:sliding_window}
\end{figure}
\begin{algorithm}[htbp!]
    \SetKwFunction{isOddNumber}{isOddNumber}
    \SetKwInOut{KwIn}{Input}
    \SetKwInOut{KwOut}{Output}
  \SetKwFunction{FMain}{Main}
    \KwIn{Time series data $(X_t)_{t=1}^{T}$, window size $n$, threshold $\eta$, peak search procedure $\mathtt{PeakSearch}$
    }
    \KwOut{Predicted change point sequence $(\hat \tau_k)$.}
    \For{$t= n, n+1, \dots, T-n$}{
        $z_t = \mathtt {GoFstat} \big(\{X_{t-n+1},... X_t\}, \{X_{t+1},..., X_{t+n}\}\big)$
    }
    $(\hat \tau_k) = \mathtt{PeakSearch} ((z_t), \eta)) $
    \caption{Sliding window-based CPD}
    \label{alg:alg1}
\end{algorithm}

Within the context of the sliding window approach we discuss several ways to quantify how close the sequence $(\hat{\tau}_k)$ is to $(\tau_j)$ which are made precise in Section \ref{sec:experiments}. Heuristically a sequence $(\hat{\tau}_k)$ is close to $(\tau_j)$ if it has the following two properties.  
\begin{enumerate}
    \item \textbf{(High True Change Point Detection Rate)} For most $\tau_j$, there should be a point $\hat{\tau}_k$ close to it. For a pre-specified tolerance $\xi > 0$, we say that the change point $\tau_j$ is detected  if there is a $\hat{\tau}_k$ such that $|\tau_j - \hat{\tau}_k| \leq \xi$, otherwise it is missed. This requires the algorithm to identify and localize changes in the distribution.
    
    \item \textbf{(Low False Alarm Rate)} For almost every $\hat{\tau}_k$ there should be a point $\tau_j$ such that $|\hat{\tau}_k - \tau_j| \leq \xi$. A predicted change point $\hat{\tau}_k$ that is far from every true change point $\tau_j$ is considered a \emph{false alarm}. This rules out algorithms which find true change points by proposing many spurious ones.
\end{enumerate}

The CPD setting described above is idealized in two important ways. First, in practice there is often no sharp threshold where a shift in the sampling distribution occurs. The distribution may undergo a short phase transition, and one would like to register this as only a single change point instead of a change point at every time step during the transition \cite{cheng2021dynamical}. Second, real data distributions may exhibit subtle fluctuations around a typical distribution, and only occasionally undergo meaningful transitions. This can make statistical tests which are \textit{too} powerful ineffective in practice because one often does not want to register small fluctuations as change points.

The discrepancy between practical and theoretical change points should inform the design of a CPD algorithm. Any CPD algorithm should be \emph{sensitive} enough to identify potentially subtle shifts and  capture all the true changes, while being \emph{robust} to insignificant fluctuations. In addition one would like to have good sample convergence properties to mitigate the impact of statistical noise in the observed samples. In the standard GoF setting, these properties can be stated in terms of the power and confidence level of the test in order to minimize false alarms and missing true change points, respectively \cite{wasserman2013all}. 

For further discussion of specific approaches to CPD see Section \ref{sec:approach_cpd} in the appendix.

\section{Background}

\subsection{Optimal Transport and Rank Energy} 
Let $\mathcal{P}(\Omega)$ denote the space of probability measures over an open set $\Omega \subseteq \mathbb{R}^d$ and let $\Pac(\Omega)$ be those measures which are absolutely continuous with respect to the Lebesgue measure on $\Omega$ (i.e. those that admit a density function). For two measures $P,Q \in \mathcal{P}(\Omega)$, the optimal transport problem with squared Euclidean ground cost seeks an optimal \emph{coupling} $\pi$ between the source distribution $P$ and the target distribution $Q$ via solving \cite{santambrogio2015optimal}
\begin{align}\label{eq:kantoro_OT}
    W_2^2(P,Q) \triangleq \min_{\pi\in \Pi(P, Q)}\int \frac{1}{2}\| x -  y \|^2 d\pi( x, y),
\end{align}
where $\Pi(P, Q)$ is the set of joint probability measure on $\mathcal P (\Omega \otimes \Omega)$ with marginals $P$ and $Q$. 

The connection between optimal transport and ranking can be understood starting with $d=1$, where when $P \in \Pac(\Omega)$ and $Q=\mbox{Unif}[0,1]$, the optimal plan is supported on $\{(x, \mbox{CDF}_P(x))\}$ \cite{santambrogio2015optimal} ($\mbox{CDF}_P(x)$ is the cumulative distribution function of $P$), which corresponds to a cyclically monotone rearrangement that in turn aligns with the natural ordering on $\mathbb{R}$. In higher dimensions, as implicitly noted in the seminal papers extending the notion of ranks to higher dimensions via optimal transport \cite{hallin2017distribution,chernozhukov2017monge,deb2021multivariate}, the \emph{key geometric} property of cyclical monotonicty is preserved in that when $P \in \Pac(\Omega)$ by the Brenier-McCann theorem \cite{brenier1991polar, mccann1995existence}, the optimal transport plans are supported on cyclically monotone sets i.e. on $\{(x, T(x)) : x\in \text{supp}(P)\}$ for a map $T:\mathbb{R}^{d}\rightarrow\mathbb{R}^{d}$, which is a gradient of a convex function (hence cyclically monotone by a well-known theorem of Rockafeller 
\cite{rockafellar1970convex}) and  satisfies $(T\#P)[A] = P[T^{-1}(A)]$ for all measurable sets $A$. 

This allows one to meaningfully interpret \emph{multivariate ranks via optimal transport maps}  as corresponding to a cyclically monotone rearrangement with respect to a target measure $Q$. Fixing the target measure $Q$ to be $\Unif$ motivates the following definition of the \textit{multivariate rank map}.
\begin{definition}[ \cite{deb2021multivariate}] \label{def:population_rank}
    Let $P\in \Pac(\Omega)$ and let $Q \!=\! \text{Unif}([0,1]^d)$. The \emph{(multivariate) rank map} of $P$ is defined as $\mathtt{R} = \nabla \phi:\mathbb{R}^{d}\rightarrow\mathbb{R}^{d}$  where $\phi$ is the convex function such that $\nabla \phi$ optimally transports $P$ to $Q$. 
\end{definition}
Using this notion of rank, rank energy is defined as follows. 
\begin{definition}[Definition 3.2, \cite{deb2021multivariate}]\label{def:re}
    Let $P_X, P_Y\in \Pac(\Omega)$ and let $X,X' \overset{i.i.d.}{\sim} P_X$ and $Y,Y'\overset{i.i.d.}{\sim} P_Y$. Let $P_\lambda = \lambda P_X +(1-\lambda)P_Y$ denote the mixture distribution for any $\lambda \in (0,1)$ and let $\R$ be the multivariate rank map of $P_\lambda$ as in Definition \ref{def:population_rank}. The \emph{(population) rank energy} is given by
    \begin{align}\label{eq:population_rank}
        \RE(P_X, P_Y)^2 \triangleq 2\mathbb{E}\norm{\R(X) - \R(Y)}
         - \mathbb{E}\norm{\R(X) - \R(X')} - \mathbb{E}\norm{\R(Y) - \R(Y')} \notag.
    \end{align}
\end{definition}
In \cite{deb2021multivariate} it is shown that the the rank energy is 1. distribution free under the null 2. consistent against alternatives (under the alternative hypothesis the probability of accepting the null hypothesis goes to zero as $n$ goes to infinity), and 3. computationally feasible for $n$ not too large. These make the rank energy a very attractive statistic for GoF since 
\subsection{Entropic Optimal Transport and Soft Rank Energy}

In order to define the soft rank energy we  begin by introducing the entropically regularized-OT (EOT) problem. 

The \textit{entropy-regularized} version of (\ref{eq:kantoro_OT}) adds an additional term to the objective \cite{cuturi2013sinkhorn,peyre2019computational,genevay2016stochastic}. For $\varepsilon>0$, the primal formulation of EOT is given by
\begin{equation}\label{eq:entropicOT}
    \min_{ \pi\in \Pi(P, Q)} \int \frac{1}{2}\|x-y\|^2 \text d\pi( x,  y) + \varepsilon \text{KL}( \pi \  || \ P \otimes Q),
\end{equation}
where
\begin{equation*}
    \text{KL}(\pi|P \otimes Q) \triangleq \int \ln \left ( \frac{d\pi( x, y)}{ dP(x) dQ( y)} \right ) d\pi( x, y).
\end{equation*}
Let $\pi_\varepsilon$ denote the solution to \eqref{eq:entropicOT}.  Extending the ideas in \cite{cuturi2019differentiable}, \cite{masud2021multivariate} proposed the following.
\begin{definition}[ \cite{masud2021multivariate}]\label{def:ent_re} Let $P \in \Pac(\Omega)$ and $Q = \Unif$. Define the \emph{entropic rank map} via $\mathtt{R}_\varepsilon(x) \triangleq \mathbb{E}_{Y \sim \pi_\varepsilon} \left [ Y | X=x \right ],$  the conditional expectation under the coupling $\pi_\varepsilon$. 
\end{definition}

\begin{remark} We note that $\mathtt{R}_\varepsilon$ is a gradient of a convex function \cite{sinho2022contraction} thereby maintaining the key geometric property of rank maps, namely cyclical monotonicity.
\end{remark}
For general measures $P,Q$, the entropic map $\Teps \triangleq \mathbb{E}_{Y \sim \pi_\varepsilon} \left [ Y | X=x \right ]$ was considered in \cite{pooladian2021entropic} and shown to be a good estimator of the unregularized optimal transport map under regularity conditions on $P,Q$.

Based on this notion, and motivated by the nicer sample and computational complexity as well as differentiability of entropic rank maps in \cite{masud2021multivariate} the following variant of rank energy was proposed and utilized for learning generative models.

\begin{definition}[Soft Rank Energy, \cite{masud2021multivariate}] \label{def:sre}
    Let $P_X,P_Y \in \Pac(\Omega)$ and let $X,X' \overset{i.i.d.}{\sim} P_X, Y,Y' \overset{i.i.d.}{\sim} P_Y$.  Let $P_\lambda = \lambda P_X +(1-\lambda) P_Y$ for  $\lambda\in (0,1)$ and let $\sR$ be the entropic rank map of $P_\lambda$. The \emph{soft rank energy} (sRE) is defined as:
    \begin{align}\label{eq:sre}
        \sRE(P_X,P_Y)^2 = 2\mathbb E\big\|\sR(X)-\sR(Y)\big\| 
         - \mathbb E\big\|\sR(X)-\sR(X')\big\| - \mathbb E\big\|\sR(Y)-\sR(Y')\big\|\notag.
    \end{align}
\end{definition}
Note that while $\sRE$ is not a metric, it is symmetric and $\sRE(P_X, P_Y) = 0$ if $P_X = P_Y$, which is useful in CPD applications.

\subsection{Estimating \texorpdfstring{$\RE$}{RE} and \texorpdfstring{$\sRE$}{sRE} from samples} \label{sec:prac_stat}
Let $X_1,...,X_n \sim P$ and $Y_1,...,Y_n \sim Q$ be jointly independent samples. Using these samples one constructs the empirical measures $P^n = \frac{1}{n}\sum_{i=1}^n \delta_{X_i}, \ Q^n = \frac{1}{n} \sum_{j=1}^n \delta_{Y_j}.$
The plug-in estimates of the OT map $\hat{T}$ and the optimal EOT coupling $\hat{\pi}_\varepsilon$ are obtained by solving 
\begin{align*}
    \hat{T} &= \argmin_{T:T\#P^n=Q^n} \frac{1}{n} \sum_{i=1}^n \norm{T(X_i) - X_i}^2, \\
\hat{\pi}_\varepsilon &= \argmin_{\pi \in \Pi(P^n, Q^n)} \sum_{i,j=1}^n \pi_{ij}\norm{X_i - Y_j}^2 + \varepsilon \pi_{ij}\log \pi_{ij}.
\end{align*}
The plug-in estimate of the entropic map $\hat{T}_\varepsilon$ is given by
\begin{equation*}
    \hat{T}_\varepsilon(X_i) = \mathbb{E}_{Y \sim \hat{\pi}_\varepsilon } \left [ Y | X = X_i \right ] = n\sum_{j=1}^n (\hat{\pi}_\varepsilon)_{ij} Y_j.
\end{equation*}
Note that like $\hat{T}$, the map $\hat{T}_\varepsilon$ is only defined on the samples $\{X_1,...,X_n\}$. 
In \cite{de2021consistent, pooladian2021entropic} these maps are extended to the whole space, and it is a novel feature of this work that \textit{no extensions are required}.
When $Q = \text{Unif}([0,1]^d)$ we say that the estimate $\hat{T}$ is the \textit{sample rank} and denote it by $\mathtt{R}_{m,n}$, with the subscript referring to the number of samples used. Analogously the estimate $\hat{T}_\varepsilon$ is referred to as the \textit{entropic sample rank} and denoted by $\sRn$.

We can now define the \textit{sample rank energy} and \textit{sample soft rank energy}.
\begin{definition} \label{def:rank_energy}
    Given two sets of samples $X_1, \dots, X_m \sim P_X$ and $Y_1, \dots, Y_n \sim P_Y$, define the empirical mixture of the two sets of samples
    $P^{m+n} = \frac{1}{m+n}\left ( \sum_{i=1}^m \delta_{X_i} + \sum_{j=1}^n \delta_{Y_j} \right )$. Let $Q^{m+n} =  \frac{1}{m+n}\sum_{i=1}^{n+m} \delta_{U_i}$ where $U_i \sim \text{Unif}([0,1]^d)$. Let $\Rn$ be the sample rank of $P^{m+n}$ to $Q^{m+n}$. The \emph{sample rank energy} is given by
    \begin{align*} 
        &\REn(P_X,P_Y)^2 \triangleq \frac{2}{mn}\sum_{i,j=1}^{m,n} \|\Rn(X_i)- \Rn(Y_j)\|
        - \frac{1}{m^2}\sum_{i,j=1}^{m}\|\Rn(X_i)- \Rn(X_j)\| \nonumber \\
        &\hspace{3cm} - \frac{1}{n^2}\sum_{i,j=1}^{n}\|\Rn(Y_i)- \Rn(Y_j)\|. 
    \end{align*}
    Let $\sRn$ be the entropic sample rank of $P^{m+n}$ to $Q^{m+n}$. The \emph{sample soft rank energy} is given by
    \begin{align*} 
        &\sREn(P_X,P_Y)^2 \triangleq \frac{2}{mn}\sum_{i,j=1}^{m,n} \|\sRn(X_i)- \sRn(Y_j)\| 
         - \frac{1}{m^2}\sum_{i,j=1}^{m}\|\sRn(X_i)- \sRn(X_j)\| \nonumber \\
        &\hspace{3cm} - \frac{1}{n^2}\sum_{i,j=1}^{n}\|\sRn(Y_i)- \sRn(Y_j)\|. 
    \end{align*}
\end{definition}
\section{Utilizing \texorpdfstring{$\sRE$}{sRE} for Change Point Detection} \label{sec:CPD_main}
The primary application of this paper is the use of $\sRE$ as a means to solve the CPD problem introduced in Section \ref{sec:cpd_problem} by using it as a GoF statistic (see Algorithm \ref{alg:alg1} and Figure \ref{fig:sliding_window}). While it is now well established that entropic OT maps can be computed much faster with practical methods that are parallelizable \cite{altschuler2017near}, computation of OT maps that requires solving a linear program still does not scale well in practice \cite{lin2022efficiency}, thereby putting use of RE at a computational disadvantage compare to sRE. In this Section, we argue that soft rank energy has several more important advantages over rank energy for CPD.

\subsection{Wasserstein Continuity Properties of Rank and Soft Rank Energy}
In \cite{deb2021multivariate} it was shown that the rank energy is distribution-free under the null hypothesis that $P_X = P_Y$. Given that the soft rank energy is ``close" the rank energy (as quantified by Theorems \ref{thm:sre_to_re_non_asymp} and \ref{thm:sre_to_re}), it is reasonable to hope that it should retain this property in an approximate sense.  While the rank energy enjoys this important theoretical property, it poses issues for CPD beyond its considerable computational and statistical burdens \cite{masud2021multivariate}. In particular, the rank energy can be highly unstable to small Wasserstein perturbations in the underlying distributions, as shown in the following theorem.
\begin{theorem} \label{thm:re_non_smooth}
    For any $\lambda \in (0,1)$ and any $\epsilon, \delta > 0$ there exists a pair of measures $P_X,P_Y$ with $W_1(P_X,P_Y) < \delta$ and 
    $$\RE(P_X, P_Y) \geq \sup_{Q_X,Q_Y\in\Pac(\Omega)} \RE(Q_X,Q_Y) - \epsilon.$$
\end{theorem}
The proof is deferred to Section \ref{sec:re_non_smooth}, and relies on an invariance property of the OT map. This result shows that the rank energy strongly distorts the Wasserstein-1 metric, in the sense that there are no universal constants $0<\alpha\le\beta<\infty$ such that
\begin{equation*}
    \alpha W_1(P_X,P_Y) \leq \RE(P_X,P_Y) \leq \beta W_1(P_X,_Y)
\end{equation*}
for any pair $P_X,P_Y$.  The nonexistence of $\beta$ follows immediately from Theorem \ref{thm:re_non_smooth}.  The nonexistence of $\alpha$ follows by taking a sequence $\{(P^i_X,P^i_Y)\}_{i=1}^{\infty}$ so that $W_1(P^i_X,P^i_Y) \rightarrow \infty$ and noting that by definition, for any $P^{i}_{X}, P^{i}_{Y}$, $\RE(P^i_X,P^i_Y) \leq 2\sqrt{d}$. At this level, the rank energy fails to properly capture a standard notion of distance between measures, and can either greatly inflate or diminish relative to Wasserstein-1.  This stands in contrast with several other common measures of similarity between probability measures \cite{polyanskiy2016wasserstein, eckstein2021quantitative}, including as we will see the soft rank energy.

We posit that the lack of an upper bound makes the rank energy overly sensitive and leads to a high false alarm rate in the CPD problem. This is because in practice there is an implicit, application-dependent threshold of distributional change which should be tolerated and not be flagged as a change point. In contrast, the rank energy aims to capture \emph{any} change, no matter how subtle, which leads to the identification of change points below the implicit threshold. The proof of Theorem \ref{thm:re_non_smooth} also suggests that the rank energy is unstable when working with distributions that are much more concentrated than $\text{Unif}([0,1]^d)$. 

In contrast, the soft rank energy enjoys a stability property with respect to $W_1$, which suggests that it is robust to small Wasserstein perturbations and may not raise a false alarm in these circumstances.

\begin{theorem} \label{thm:sre_smooth} For any $\lambda \in (0,1)$ and $P_X,P_Y \in \Pac(\mathbb{R}^d)$ it holds
\begin{equation*}
    \sRE(P_X, P_Y)^2 \leq \frac{2d}{\varepsilon}W_1(P_X, P_Y).
\end{equation*}
\end{theorem}

The proof is deferred to Section \ref{sec:sre_smooth} and relies crucially on the Lipschitz continuity of the entropic map (see Lemma \ref{lem:bounded_lip}).

\begin{remark}
    In Theorem \ref{thm:sre_smooth}, the factor $2d$ in the bound is an artifact of using $Q = \text{Unif}([0,1]^d).$ If instead one chose $Q = \text{Unif}(B_2^d(u,1))$ for any $u \in \mathbb{R}^d,$ then the bound above could be replaced by a dimension-free 8. Additionally, since $W_1(P_X,P_Y) \leq W_p(P_X,P_Y)$ for all $p \geq 1$ the conclusion also holds for these variants of the Wasserstein distance. We state it in terms of $W_1$ since it is the strongest bound of this form. 
\end{remark}

Comparing Theorem \ref{thm:re_non_smooth} to Theorem \ref{thm:sre_smooth}, there is a clear qualitative difference between the rank energy and the soft rank energy. This sensitivity also appears empirically and is demonstrated in Figure \ref{fig:re_vs_sre}. In this figure when the samples are highly concentrated the rank energy suffers from large fluctuations while the soft rank energy remains comparatively smooth. This leads the $\re$ to produce a few false positives while the $\sre$ shows stability against those fluctuations.
\begin{figure*}
    \centering
    \includegraphics[width = \linewidth]{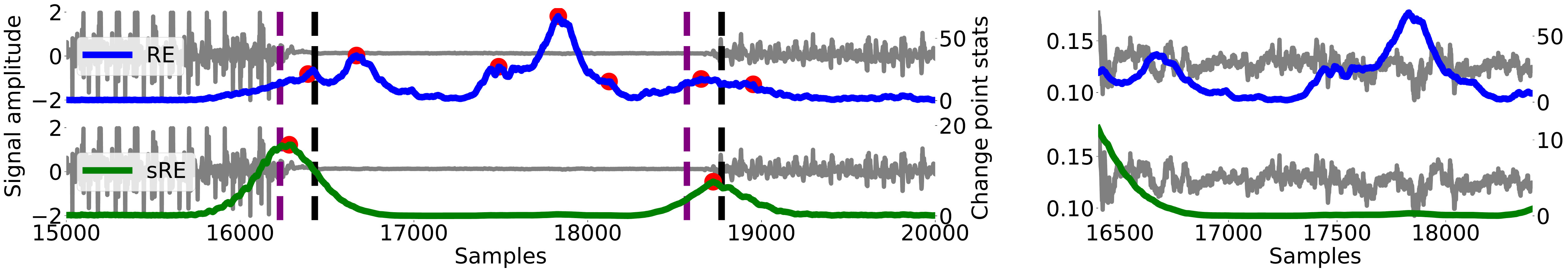}
    \vspace{-8mm}
    \caption{The left plot shows a sequence of HASC-PAC2016 dataset  with the detected change points
marked by red dots. The right plot is a zoomed-in version of a short segment. Both
RE and sRE can detect the true changes (black dashed line) within
a certain margin (dashed purple), but RE also produces false positives due to sensitivity to small signal fluctuations. On the other hand, sRE displays greater stability in this aspect, leading to superior performance as seen in Table \ref{tab:real_result}.}
    \label{fig:re_vs_sre}
\end{figure*}
Theorem \ref{thm:sre_smooth} also suggests the role of $\varepsilon$ may act as a sensitivity knob with small $\varepsilon$ leading to a highly sensitive signal while a large $\varepsilon$ is more stable against perturbations. 

\subsection{Convergence of \texorpdfstring{$\sRE$}{sRE} to \texorpdfstring{$\RE$}{RE}}
While Theorems \ref{thm:re_non_smooth} and \ref{thm:sre_smooth} suggest that there may be some fundamental differences between the soft rank energy and the rank energy, it is still possible to derive convergence results between them. This is to be expected since the optimal entropic coupling, $\pi_\varepsilon$ between $P \in \Pac(\Omega)$ and $Q$ is known to converge weakly to the unregularized coupling $\pi = [\text{Id} \otimes T]\#P$ (where $T$ is the OT map from $P$ to $Q$) as $\varepsilon \rightarrow 0$. An asymptotic result demonstrating this is given in Theorem \ref{thm:sre_to_re}. If one imposes further assumptions on the OT map, namely Lipschtiz continuity, one may use the recent results from \cite{carlier2022convergence} to arrive at a \emph{quantitative} estimate of the difference between $\sRE$ and $\RE$. The proof is deferred to Section \ref{sec:sre_to_re_non_asymp}. 

\begin{theorem} \label{thm:sre_to_re_non_asymp} 
    Under assumptions of compactness of domains and $L$-Lipschitz continuity of $\R$, it holds that $|\sRE(P_X,P_Y)^2 - \RE(P_X,P_Y)^2| \leq C
        L \sqrt{d \varepsilon \log (1/\varepsilon) + O(\varepsilon)},$ for some constant $C$.
\end{theorem}
 
Theorem \ref{thm:sre_to_re_non_asymp} 
implies that for small $\varepsilon$ the soft rank energy is a close approximation of the rank energy. This is important because the rank energy is distribution-free under the null hypothesis that $P=Q$.  It is reasonable to expect that the soft rank energy with small $\varepsilon$ approximately inherits this property; this is empirically observed in Figure \ref{fig:null_stat}, where one can see that the soft rank energy is nearly distribution-free under the null.
\begin{figure}[ht]
    \centering
    \includegraphics[width = \linewidth]{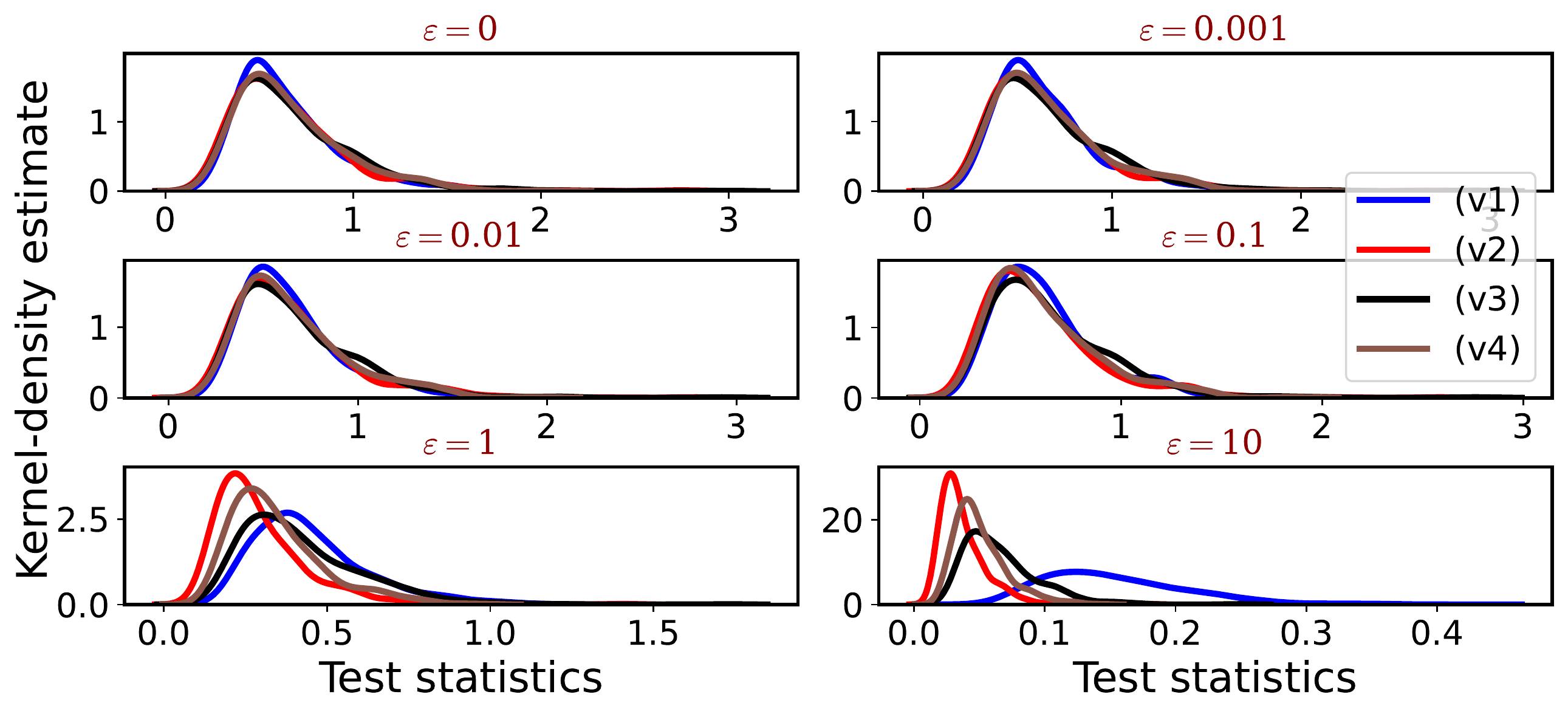}
    \vspace{-8mm}
    \caption{Kernel density estimates of RE $(\varepsilon = 0)$ and sRE under the null for v1 (Cauchy), v2 (multivariate Gaussian), v3 (multivariate Gaussian with diagonal covariance), and v4 (Laplace) distributional settings scaled by a factor of $mn/(m+n)$. RE exhibits distribution-free behavior under the null, and sRE shows qualitatively similar behavior for small values of $\varepsilon$. However, for larger values of $\varepsilon$, the density curves of sRE deviate from this pattern, indicating a loss of the distribution-freeness property. Here $m=n=200$ and the statistics are plotted using $1000$ random draws.}
    \label{fig:null_stat}
\end{figure}

\begin{remark}
    The convergence of $\sRE$ to $\RE$ in the presence of Theorems \ref{thm:re_non_smooth} and \ref{thm:sre_smooth} may seem suprising since these results suggest that $\re$ and $\sre$ are fundamentally different.  This can be reconciled by noting one the bound $\sRE(P_X,P_Y)^2 \leq 2\sqrt{d}$ for any $P_X,P_Y$ and $\varepsilon$ and by choosing $\varepsilon$ small enough relative to $W_1(P_X,P_Y)$ it will hold that $W_1(P_X,P_Y) / \varepsilon > 2\sqrt{d}$, which leads to the bound in Theorem \ref{thm:sre_smooth} becoming vacuous.
\end{remark}

\subsection{Statistical Properties}

Sadly, the plug-in estimate of the optimal map suffers from the curse of dimensionality.  When working in $\mathbb{R}^d$ and in the absence of further assumptions, the plug-in estimate $\hat{T}$ may converge to the true map $T$ as slowly as $n^{-1/d}$ \cite{dudley1969speed}. In fact, \cite{Hutter2021_Minimax} show that for \textit{any (measurable) estimator} $T_0$ there exists a measure $P$ with
\begin{equation*}
    \mathbb{E} \int_{\mathbb{R}^{d}} \norm{T_0(x) - T(x)}^2 dP(x) \gtrsim n^{-2/d}.
\end{equation*}
This says that OT truly does suffer from the curse of dimensionality unless further assumptions are placed on the measures $P$ and $Q$. Practically, the poor statistical convergence rates coupled with the computational problems when working with large sample sizes makes standard map estimation difficult to accurately perform on high-dimensional data. In turn, $\REn(P_X,P_Y)^2$ may also converge very slowly to $\RE(P_X,P_Y)^2$. The $n^{-1/d}$ or $n^{-2/d}$ rate is typical for quantities associated to OT and this issue is often alleviated by imposing additional structural conditions on the measures, most often some type of smoothness condition \cite{shirdhonkar2008approximate, weed2019minimax}.

In stark contrast, once entropy regularization is introduced the statistical and computational issues are largely avoided. In \cite{masud2021multivariate} it is shown that when $P$ is sub-gaussian and $Q$ has bounded support then
\begin{equation*}
    \mathbb{E}||\Teps^{n,n} - \Teps||_{L^2(P)}^2 \lesssim \log(n)n^{-1/2}.
\end{equation*}
When both $P$ and $Q$ have bounded support it is shown in \cite{rigollet2022sample} that 
\begin{equation*}
    \mathbb{E}||\Teps^{n,n} - \Teps||_{L^2(P)}^2 \lesssim n^{-1}.
\end{equation*}
These fast rates of convergence are typical of entropy regularized optimal transport \cite{genevay2019sample, mena2019statistical}. It is also possible to show that $\sREn$ has a fast rate of convergence to $\sRE$, even when the dimension is large. This is stated in the following result.
\begin{theorem} \label{thm:sren_to_sre} Let $P_X,P_Y \in \mathcal{P}(B(0,r))$. Let $X_1,...,X_n \sim P_X$ and $Y_1,...,Y_n \sim P_Y$ be jointly independent. Then 
    \begin{align*}
    \mathbb{E}|\mathtt{sRE}_{n,n}^\varepsilon(P_X,P_Y)^2 - \mathtt{sRE}_{1/2}^\varepsilon(P_X,P_Y)^2| 
        \leq \frac{24r\sqrt{1+\varepsilon^2}}{\sqrt{2n}}\exp(22r^2/\varepsilon) + 8 \sqrt{\frac{d\pi}{n}}.
    \end{align*}
\end{theorem}
The proof is deferred to Section \ref{sec:sren_to_sre}. The first step in proving this result is to introduce an intermediary term which approximates the $\sRE$ by a discrete sum evaluated at the sample points $X_1,...,X_n,Y_1,...,Y_n$ and then apply the triangle inequality. This breaks the estimate into two terms, the first a Monte Carlo estimate of an expectation (which is easily controlled), the second measuring how closely $\sRn$ approximates $\sR$ on the sampled points. Controlling the second term  is substantially complicated by two things. First the distribution of $(X_1,...,X_m,Y_1,...,Y_n)$ is \textit{not the same} as $(P_\lambda)^{m+n}$. Second there is a dependence between the sample entropic rank $\sRn$ and the points used to estimate it, and one must ensure that on these points the soft rank map is well behaved. In \cite{masud2021multivariate} these issues are handled via resampling ideas, however this approach wastes samples, requires artificially sampling from $P_\lambda$, and requires an out-of-sample extension of the soft rank map since $\sRn$ is only defined at the sample points. However our approach is limited to measures with compact support while \cite{masud2021multivariate} are able to cover subgaussian measures, a question we leave to future work.

\section{Numerical Evaluation} \label{sec:experiments}
The main hyperparameters for Algorithm \ref{alg:alg1} are the window size $n$ and threshold parameter $\eta$. While evaluating CPD on synthetic data, we study the effect of $n$ on performance. On the other hand, for real-world data, we use our domain knowledge of the typical frequency of change points to set the window size appropriately. Once we have calculated the change point statistics, we use a standard peak finding procedure\footnote{We use \href{https://docs.scipy.org/doc/scipy/reference/generated/scipy.signal.find_peaks.html}{scipy.signal.find\_peaks} from Python Scipy1.9.1.} with thresholding to identify potential change points. Since the statistical guarantees of the various tests differ, we evaluate their performance using metrics that vary the threshold parameter $\eta$ over all possible values\footnote{Code to reproduce results are available at \url{https://github.com/ShoaibBinMasud/CPD-using-sRE}}. 

One potential drawback of the peak search algorithm is that it may generate many small sub-peaks around the largest peaks. To prevent the detection of multiple successive change points when only one change point is present, we apply a minimal horizontal distance $\Delta$ in samples to ensure that every pair
of predicted change points $\hat \tau\neq \hat \tau'$ are at least $\Delta$ samples apart.
\subsection{Evaluation Metrics}
We consider two widely used metrics in CPD literature \cite{aminikhanghahi2017survey, cheng2020optimal} to evaluate the performance, (a) area under the precision-recall curve, (b) best F1-score across all detection thresholds. The F1-score is defined as:
\begin{align*}
\text{F1-score} = \frac{2\cdot \text{precision} \times \text{recall}}{\text{precision} + \text{recall}},\hspace{0.5cm}\\ \text{precision} = \frac{TP}{TP+ FP},\hspace{0.5cm} \text{recall} =\frac{TP}{TP+ FN},
\end{align*}
where TP, FP, and FN represent the total number of true positive, false positive, and false negative points, respectively. To account for uncertainty in the exact annotation of true change points, we allow a margin of error $\xi$ when declaring a point either as TP or FP or FN. A predicted change point $\hat \tau_k$ is considered a TP if it is within $\xi$ of a true change point $\tau_j$ (i.e., $|\tau_j - \hat{\tau}_k|\leq \xi$), otherwise it is considered a FP. A true change point $\tau_j$ that does not have a detected change within $\xi$ is considered a FN. The choice of $\delta$ is important for proper performance assessment. A small $\xi$ may increase the number of FPs, while a larger $\xi$ may misleadingly improve performance by considering detected change points far from true change points as TPs. Additionally, multiple true change points in close proximity may increase ambiguity when using a larger $\xi$. To ensure fairness in comparison, we use the same $\xi$ for all methods.
\begin{table*}[ht]
\small
\setlength{\tabcolsep}{1pt}
\centering
\caption{Performance comparison of RE and sRE with other statistics used for CPD(bold: best).}
\label{tab:real_result}
\begin{tabular}{c|ccccc|c|ccccc|c}
\hline
\multirow{2}{*}{Method} & \multicolumn{5}{c|}{AUC-PR} & \multirow{2}{*}{Average} & \multicolumn{5}{c|}{Best F1-score} & \multirow{2}{*}{Average} \\ \cline{2-6} \cline{8-12} 
 & \multicolumn{1}{c|}{ \makecell{HSAC\\PAC2016} } & \multicolumn{1}{c|}{ \makecell{HSAC\\2011} } & \multicolumn{1}{c|}{ Beedance} & \multicolumn{1}{c|}{ Salinas} & { ECG} &  & \multicolumn{1}{c|}{ \makecell{HSAC\\PAC2016} } & \multicolumn{1}{c|}{ \makecell{HSAC\\2011} } & \multicolumn{1}{c|}{ Beedance} & \multicolumn{1}{c|}{ Salinas} & { ECG} &  \\ \hline \hline
M-stat \cite{li2015scan}& \multicolumn{1}{c|}{0.688} & \multicolumn{1}{c|}{0.565} & \multicolumn{1}{c|}{0.566} & \multicolumn{1}{c|}{0.471} & 0.442 & 0.546 & \multicolumn{1}{c|}{0.804} & \multicolumn{1}{c|}{0.676} & \multicolumn{1}{c|}{0.723} & \multicolumn{1}{c|}{0.708} & 0.667 & 0.716 \\ 
SinkDiv\cite{ahad2022learning} & \multicolumn{1}{c|}{0.679} & \multicolumn{1}{c|}{0.578} & \multicolumn{1}{c|}{\bf 0.764} & \multicolumn{1}{c|}{0.501} & {\bf 0.487} & 0.601 & \multicolumn{1}{c|}{0.791} & \multicolumn{1}{c|}{0.699} & \multicolumn{1}{c|}{\bf 0.823} & \multicolumn{1}{c|}{0.558} & 0.682 & 0.710 \\ 
W1\cite{cheng2020on} & \multicolumn{1}{c|}{0.678} & \multicolumn{1}{c|}{\bf 0.652} & \multicolumn{1}{c|}{0.763} & \multicolumn{1}{c|}{0.252} & 0.441 & 0.557 & \multicolumn{1}{c|}{0.806} & \multicolumn{1}{c|}{0.702} & \multicolumn{1}{c|}{0.820} & \multicolumn{1}{c|}{0.525} & 0.682 & 0.707 \\ 
WQT\cite{cheng2020optimal} & \multicolumn{1}{c|}{0.638} & \multicolumn{1}{c|}{0.411} & \multicolumn{1}{c|}{0.424} & \multicolumn{1}{c|}{0.308} & 0.449 & 0.446 & \multicolumn{1}{c|}{0.772} & \multicolumn{1}{c|}{0.636} & \multicolumn{1}{c|}{0.698} & \multicolumn{1}{c|}{0.598} & 0.682 & 0.677 \\ 
RE & \multicolumn{1}{c|}{0.596} & \multicolumn{1}{c|}{0.382} & \multicolumn{1}{c|}{0.367} & \multicolumn{1}{c|}{0.312} & 0.482 & 0.427 & \multicolumn{1}{c|}{0.779} & \multicolumn{1}{c|}{0.641} & \multicolumn{1}{c|}{0.646} & \multicolumn{1}{c|}{0.523} & 0.684 & 0.654 \\ 
sRE & \multicolumn{1}{c|}{\bf 0.740} & \multicolumn{1}{c|}{0.598} & \multicolumn{1}{c|}{0.687} & \multicolumn{1}{c|}{\bf 0.714} & 0.473 & {\bf 0.647} & \multicolumn{1}{c|}{\bf 0.831} & \multicolumn{1}{c|}{\bf 0.709} & \multicolumn{1}{c|}{0.801} & \multicolumn{1}{c|}{\bf 0.772} & 0.682 & {\bf 0.756} \\ \hline
\end{tabular}
\end{table*}
\subsection{Results}

In this study, we evaluate and compare the performance of the sRE method with other GoF statistics for CPD on a synthetic dataset, as well as on 5 real-world datasets including 4 time-series datasets and a hyperspectral image dataset.  Detailed descriptions of the datasets as well as discussion of the various hyperparameters used in this study can be found in Section \ref{sec:datasets} in the appendix.

\begin{figure}[t]
    \centering
    \includegraphics[width =\linewidth]{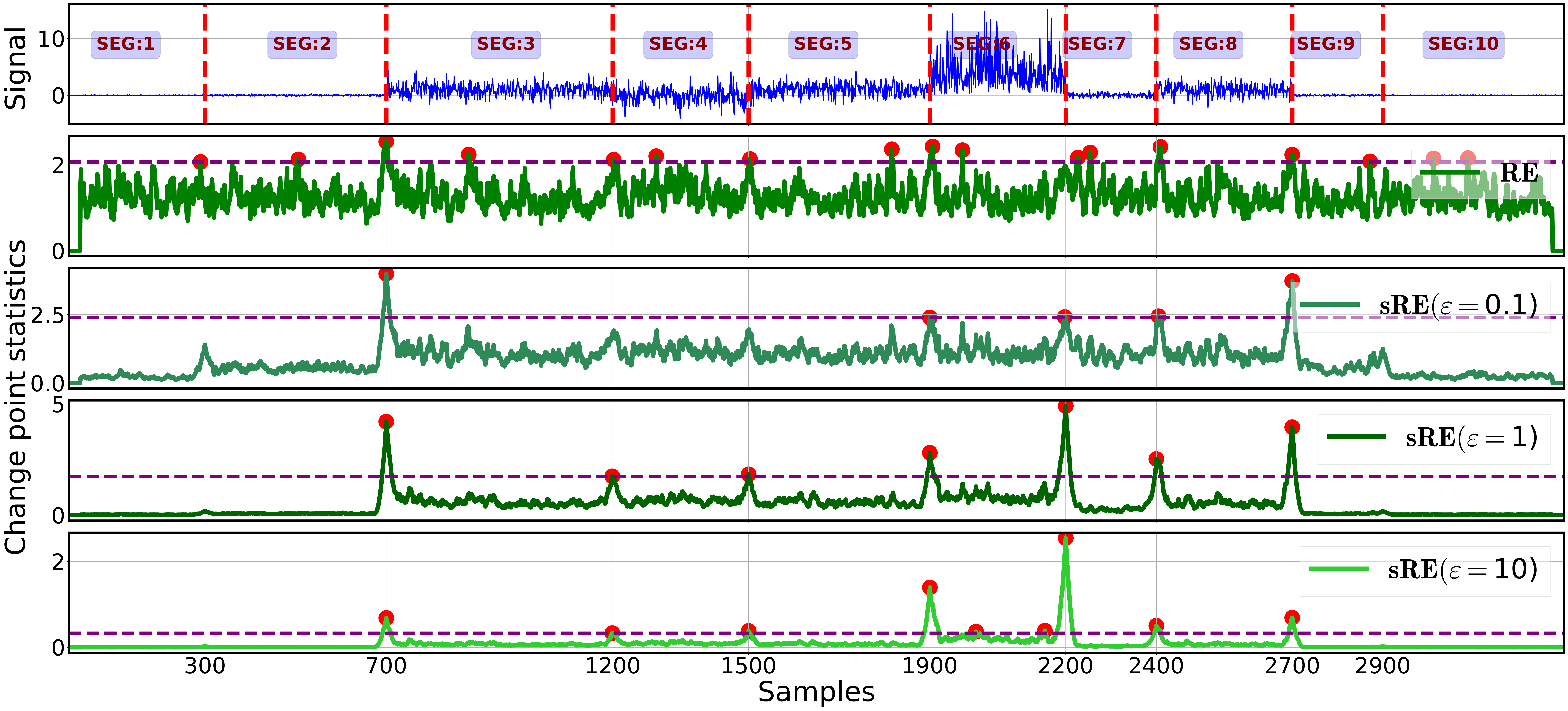}
    \vspace{-5mm}
    \caption{
    \textbf{Top Row}: A single dimension of the synthetic data (top row) with true change points (vertical dotted red line). \textbf{Bottom Four Rows}: Change point statistics (with window size $n = 25$) using RE and sRE on synthetic dataset with threshold $\eta$ (horizontal dashed purple) providing the best F1-score, the detected change points (red dot). The plot shows that sRE statistics become smoother as the value of $\varepsilon$ increases}
    \label{fig:toy_dataset}
\end{figure}
\paragraph{Synthetic:} In Figure \ref{fig:toy_dataset} we compare $\RE$  and $\sRE$ with various choices of regularization on a synthetic dataset. In this figure we see that increasing the regularization parameter $\varepsilon$ of $\sRE$ does indeed produce a smoothing effect on the generated signal in agreement with Theorem \ref{thm:sre_smooth} which in turn leads to fewer false alarms. In contrast $\RE$ is highly oscillatory and creates many false alarms which we believe is because of its poor continuity properties as discussed in Theorem \ref{thm:re_non_smooth}. However, \textit{over-regularizing} leads to missed change points due to the bound $W_1/\varepsilon$ becoming small yet still dominating $\sRE$. In the appendix both an additional figure illustrating the effect of the window size (Figure \ref{fig:toy_dataset_2col}) and numeric results comparing against other statistics are given (Table \ref{tab:synthetic_result}) are given.

\paragraph{HASC-PAC2016, HASC-2011:} On HASC-PAC2016, sRE performs the best on all metrics. On HASC2011, sRE also performs better overall compared to most of the methods. In contrast, RE has the lowest overall performance on both datasets. This is because RE produces false positives in low amplitude regions between activities, called ``rest," due to its sensitivity to any signal regardless of its amplitude (Figure \ref{fig:re_vs_sre}). In contrast, sRE provides smoother statistics compared to RE which is validated by Theorem \ref{thm:sre_smooth} and ignores changes in those regions, resulting in a significant improvement in performance.

\paragraph{Bee Dance:}
Beedance is a comparatively challenging dataset for CPD due it frequent fluctuations. Among the methods tested, Sinkdiv and W1 demonstrated the best performance in terms of AUC-PR and F1-score. While sRE also performs well, it does not achieve the same level of success as Sinkdiv and W1. In contrast, RE has the poorest overall performance, likely due to its tendency to respond to all fluctuations, including those that may not be considered as change points.

\paragraph{Salinas A:} On this high-dimensional hyperspectral image dataset, sRE outperforms all other methods by a significant margin in both AUC-PR and F1-score. This demonstrates that sRE is better able to detect the true change points with a small margin of error $(\xi = 2)$. Moreover, the results from Section \ref{sec:prac_stat} suggest that sRE is more easily estimable in high dimension compared to RE, which may be a contributing factor to its relatively strong performance on this dataset.

\paragraph{ECG:} 
Despite being developed based on the concept of multivariate rank, both RE and sRE are effective at detecting change points even when the signal is one-dimensional. As shown in Table \ref{tab:real_result}, all methods, including RE and sRE, perform similarly well on the univariate ECG signal.

\paragraph{Overall:} sRE performs better than RE and is either superior or comparable to other methods on all datasets. On the high-dimensional hyperspectral image dataset, sRE outperforms all other methods, while performing competitively on low-dimensional and even on one-dimensional data. In addition, sRE achieves the best average AUC-PR and the best average F1-score across all datasets, making it a strong candidate for GoF statistic to be used in a  sliding-window based offline and unsupervised CPD method.

\section{Conclusion and Future Work} \label{sec:conclusion}
We have established that soft rank energy enjoys efficient statistical and computational complexity, is Lipschitz with respect to Wasserstein-1, and performs well as a GoF measure on a range of real-world CPD problems. However these considerations are all made under compactness assumptions on all the measures involved. A problem left to future work is to extend these results to measures with unbounded support under certain concentration assumptions, namely the subgaussian or subexponential distributions. Results on entropic optimal transport under these assumptions exist in the literature \cite{mena2019statistical,masud2021multivariate} but do not appear to be able to directly applicable to the soft rank energy.

In addition, while we have chosen the uniform distribution on the unit cube $[0,1]^{d}$ as the target measure for the rank maps in this paper, it is of interest to consider the role of this distribution and if other distributions may allow for better convergence bounds (see discussion following Theorem \ref{thm:sre_to_re_non_asymp}).  Noting that the rank maps allow for comparisons of distributions vis-\`{a}-vis their transport maps to a specified target distribution, it is of interest to investigate the complimentary picture namely comparing distributions via their multivariate quantile maps \cite{hallin2017distribution, chernozhukov2017monge}. and connections with the \emph{linear optimal transport} framework \cite{wang2013linear}, where one compares the distributions via the transport maps from a specific reference measure to these distributions as the target measures.

\bibliographystyle{abbrv} 
\bibliography{icml_axiv.bib}

\begin{thebibliography}{10}

\bibitem{adams2007bayesian}
R.~P. Adams and D.~J. MacKay.
\newblock Bayesian online changepoint detection.
\newblock {\em arXiv preprint arXiv:0710.3742}, 2007.

\bibitem{ahad2022learning}
N.~Ahad, E.~L. Dyer, K.~B. Hengen, Y.~Xie, and M.~A. Davenport.
\newblock Learning {S}inkhorn divergences for supervised change point
  detection.
\newblock {\em arXiv preprint arXiv:2202.04000}, 2022.

\bibitem{altschuler2017near}
J.~Altschuler, J.~Niles-Weed, and P.~Rigollet.
\newblock Near-linear time approximation algorithms for optimal transport via
  {S}inkhorn iteration.
\newblock In I.~Guyon, U.~V. Luxburg, S.~Bengio, H.~Wallach, R.~Fergus,
  S.~Vishwanathan, and R.~Garnett, editors, {\em Advances in Neural Information
  Processing Systems}, volume~30. Curran Associates, Inc., 2017.

\bibitem{aminikhanghahi2017survey}
S.~Aminikhanghahi and D.~J. Cook.
\newblock A survey of methods for time series change point detection.
\newblock {\em Knowledge and information systems}, 51(2):339--367, 2017.

\bibitem{anderson1962distribution}
T.~W. Anderson.
\newblock On the distribution of the two-sample cramer-von mises criterion.
\newblock {\em The Annals of Mathematical Statistics}, pages 1148--1159, 1962.

\bibitem{bernton2021entropic}
E.~Bernton, P.~Ghosal, and M.~Nutz.
\newblock Entropic optimal transport: Geometry and large deviations.
\newblock {\em arXiv preprint arXiv:2102.04397}, 2021.

\bibitem{brenier1991polar}
Y.~Brenier.
\newblock Polar factorization and monotone rearrangement of vector-valued
  functions.
\newblock {\em Communications on Pure and Applied Mathematics}, 44(4):375--417,
  1991.

\bibitem{carlier2022convergence}
G.~Carlier, P.~Pegon, and L.~Tamanini.
\newblock Convergence rate of general entropic optimal transport costs.
\newblock {\em arXiv preprint arXiv:2206.03347}, 2022.

\bibitem{chamroukhi2013joint}
F.~Chamroukhi, S.~Mohammed, D.~Trabelsi, L.~Oukhellou, and Y.~Amirat.
\newblock Joint segmentation of multivariate time series with hidden process
  regression for human activity recognition.
\newblock {\em Neurocomputing}, 120:633--644, 2013.

\bibitem{chang2019kernel}
W.-C. Chang, C.-L. Li, Y.~Yang, and B.~P{\'o}czos.
\newblock Kernel change-point detection with auxiliary deep generative models.
\newblock {\em arXiv preprint arXiv:1901.06077}, 2019.

\bibitem{cheng2021dynamical}
K.~Cheng, S.~Aeron, M.~C. Hughes, and E.~L. Miller.
\newblock Dynamical {W}asserstein barycenters for time-series modeling.
\newblock {\em Advances in Neural Information Processing Systems},
  34:27991--28003, 2021.

\bibitem{cheng2020optimal}
K.~C. Cheng, S.~Aeron, M.~C. Hughes, E.~Hussey, and E.~L. Miller.
\newblock Optimal transport based change point detection and time series
  segment clustering.
\newblock In {\em ICASSP 2020-2020 IEEE International Conference on Acoustics,
  Speech and Signal Processing (ICASSP)}, pages 6034--6038. IEEE, 2020.

\bibitem{cheng2020on}
K.~C. {Cheng}, E.~L. {Miller}, M.~C. {Hughes}, and S.~{Aeron}.
\newblock On matched filtering for statistical change point detection.
\newblock {\em IEEE Open Journal of Signal Processing}, 1:159--176, 2020.

\bibitem{chernozhukov2017monge}
V.~Chernozhukov, A.~Galichon, M.~Hallin, and M.~Henry.
\newblock Monge--{K}antorovich depth, quantiles, ranks and signs.
\newblock {\em Annals of Statistics}, 45(1):223--256, 2017.

\bibitem{sinho2022contraction}
S.~Chewi and A.-A. Pooladian.
\newblock An entropic generalization of {C}affarelli's contraction theorem via
  covariance inequalities, 2022.

\bibitem{cordero2019regularity}
D.~Cordero-Erausquin and A.~Figalli.
\newblock Regularity of monotone transport maps between unbounded domains.
\newblock {\em Dynamical Systems}, 39(12):7101--7112, 2019.

\bibitem{cramer1928composition}
H.~Cram{\'e}r.
\newblock {\em On the composition of elementary errors: Statistical
  applications}.
\newblock Almqvist and Wiksell, 1928.

\bibitem{cuturi2013sinkhorn}
M.~Cuturi.
\newblock Sinkhorn distances: Lightspeed computation of optimal transport.
\newblock {\em Advances in neural information processing systems},
  26:2292--2300, 2013.

\bibitem{cuturi2019differentiable}
M.~Cuturi, O.~Teboul, and J.-P. Vert.
\newblock Differentiable ranks and sorting using optimal transport.
\newblock {\em arXiv preprint arXiv:1905.11885}, 2019.

\bibitem{damjanovic2021catboss}
J.~Damjanovic, J.~M. Murphy, and Y.-S. Lin.
\newblock Catboss: Cluster analysis of trajectories based on segment splitting.
\newblock {\em Journal of Chemical Information and Modeling},
  61(10):5066--5081, 2021.

\bibitem{de2021consistent}
L.~De~Lara, A.~Gonz{\'a}lez-Sanz, and J.-M. Loubes.
\newblock A consistent extension of discrete optimal transport maps for machine
  learning applications.
\newblock {\em arXiv preprint arXiv:2102.08644}, 2021.

\bibitem{deb2021efficiency}
N.~Deb, B.~B. Bhattacharya, and B.~Sen.
\newblock Efficiency lower bounds for distribution-free hotelling-type
  two-sample tests based on optimal transport.
\newblock {\em arXiv preprint arXiv:2104.01986}, 2021.

\bibitem{deb2021multivariate}
N.~Deb and B.~Sen.
\newblock Multivariate rank-based distribution-free nonparametric testing using
  measure transportation.
\newblock {\em Journal of the American Statistical Association}, pages 1--16,
  2021.

\bibitem{dudley1969speed}
R.~M. Dudley.
\newblock {The Speed of Mean Glivenko-Cantelli Convergence}.
\newblock {\em The Annals of Mathematical Statistics}, 40(1):40 -- 50, 1969.

\bibitem{eckstein2021quantitative}
S.~Eckstein and M.~Nutz.
\newblock Quantitative stability of regularized optimal transport and
  convergence of sinkhorn’s algorithm.
\newblock {\em arXiv preprint arXiv:2110.06798}, 2021.

\bibitem{feydy2019interpolating}
J.~Feydy, T.~S{\'e}journ{\'e}, F.-X. Vialard, S.-i. Amari, A.~Trouv{\'e}, and
  G.~Peyr{\'e}.
\newblock Interpolating between optimal transport and {MMD} using {S}inkhorn
  divergences.
\newblock In {\em The 22nd International Conference on Artificial Intelligence
  and Statistics}, pages 2681--2690. PMLR, 2019.

\bibitem{genevay2019sample}
A.~Genevay, L.~Chizat, F.~Bach, M.~Cuturi, and G.~Peyr{\'e}.
\newblock Sample complexity of {S}inkhorn divergences.
\newblock In {\em The 22nd International Conference on Artificial Intelligence
  and Statistics}, pages 1574--1583. PMLR, 2019.

\bibitem{genevay2016stochastic}
A.~Genevay, M.~Cuturi, G.~Peyr{\'e}, and F.~Bach.
\newblock Stochastic optimization for large-scale optimal transport.
\newblock {\em Advances in neural information processing systems}, 29, 2016.

\bibitem{gretton2012kernel}
A.~Gretton, K.~M. Borgwardt, M.~J. Rasch, B.~Sch{\"o}lkopf, and A.~Smola.
\newblock A kernel two-sample test.
\newblock {\em The Journal of Machine Learning Research}, 13(1):723--773, 2012.

\bibitem{hallin2017distribution}
M.~Hallin.
\newblock On distribution and quantile functions, ranks and signs in
  $\mathbb{R}^{d}$.
\newblock {\em ECARES Working Papers}, 2017.

\bibitem{hallin2022measure}
M.~Hallin.
\newblock Measure transportation and statistical decision theory.
\newblock {\em Annual Review of Statistics and Its Application}, 9(1):401--424,
  2022.

\bibitem{hawkins2010nonparametric}
D.~M. Hawkins and Q.~Deng.
\newblock A nonparametric change-point control chart.
\newblock {\em Journal of Quality Technology}, 42(2):165--173, 2010.

\bibitem{he2006nonparametric}
T.~He, S.~Ben-David, and L.~Tong.
\newblock Nonparametric change detection and estimation in large-scale sensor
  networks.
\newblock {\em IEEE transactions on signal processing}, 54(4):1204--1217, 2006.

\bibitem{Hutter2021_Minimax}
J.-C. H\"{u}tter and P.~Rigollet.
\newblock {Minimax estimation of smooth optimal transport maps}.
\newblock {\em The Annals of Statistics}, 49(2):1166 -- 1194, 2021.

\bibitem{kolmogorov1933sulla}
A.~Kolmogorov.
\newblock Sulla determinazione empirica di una lgge di distribuzione.
\newblock {\em Inst. Ital. Attuari, Giorn.}, 4:83--91, 1933.

\bibitem{lee2018time}
W.-H. Lee, J.~Ortiz, B.~Ko, and R.~Lee.
\newblock Time series segmentation through automatic feature learning.
\newblock {\em arXiv preprint arXiv:1801.05394}, 2018.

\bibitem{li2015scan}
S.~Li, Y.~Xie, H.~Dai, and L.~Song.
\newblock Scan $ b $-statistic for kernel change-point detection.
\newblock {\em arXiv preprint arXiv:1507.01279}, 2015.

\bibitem{lin2022efficiency}
T.~Lin, N.~Ho, and M.~I. Jordan.
\newblock On the efficiency of entropic regularized algorithms for optimal
  transport.
\newblock {\em Journal of Machine Learning Research}, 23(137):1--42, 2022.

\bibitem{marino2020optimal}
S.~D. Marino and A.~Gerolin.
\newblock An optimal transport approach for the schr{\"o}dinger bridge problem
  and convergence of sinkhorn algorithm.
\newblock {\em Journal of Scientific Computing}, 85(2):1--28, 2020.

\bibitem{massey1951kolmogorov}
F.~J. Massey~Jr.
\newblock The {K}olmogorov-{S}mirnov test for goodness of fit.
\newblock {\em Journal of the American statistical Association},
  46(253):68--78, 1951.

\bibitem{masud2021multivariate}
S.~B. Masud, M.~Werenski, J.~M. Murphy, and S.~Aeron.
\newblock Multivariate soft rank via entropic optimal transport: sample
  efficiency and generative modeling.
\newblock {\em arXiv:2111.00043}, 2021.

\bibitem{mccann1995existence}
R.~J. McCann.
\newblock Existence and uniqueness of monotone measure-preserving maps.
\newblock {\em Duke Mathematical Journal}, 80(2):309--324, 1995.

\bibitem{mcdiarmid1989method}
C.~McDiarmid et~al.
\newblock On the method of bounded differences.
\newblock {\em Surveys in combinatorics}, 141(1):148--188, 1989.

\bibitem{mena2019statistical}
G.~Mena and J.~Niles-Weed.
\newblock Statistical bounds for entropic optimal transport: Sample complexity
  and the central limit theorem.
\newblock {\em Advances in Neural Information Processing Systems}, 32, 2019.

\bibitem{mises2013probability}
R.~v. Mises.
\newblock {\em probability statistic and truth}, volume~7.
\newblock Springer-Verlag, 2013.

\bibitem{weed2019minimax}
J.~Niles-Weed and Q.~Berthet.
\newblock Minimax estimation of smooth densities in wasserstein distance, 2019.

\bibitem{peyre2019computational}
G.~Peyr{\'e} and M.~Cuturi.
\newblock Computational optimal transport: With applications to data science.
\newblock {\em Foundations and Trends{\textregistered} in Machine Learning},
  11(5-6):355--607, 2019.

\bibitem{polyanskiy2016wasserstein}
Y.~Polyanskiy and Y.~Wu.
\newblock Wasserstein continuity of entropy and outer bounds for interference
  channels.
\newblock {\em IEEE Transactions on Information Theory}, 62(7):3992--4002,
  2016.

\bibitem{pooladian2021entropic}
A.-A. Pooladian and J.~Niles-Weed.
\newblock Entropic estimation of optimal transport maps.
\newblock {\em arXiv preprint arXiv:2109.12004}, 2021.

\bibitem{qi2014novel}
J.-P. Qi, Q.~Zhang, Y.~Zhu, and J.~Qi.
\newblock A novel method for fast change-point detection on simulated time
  series and electrocardiogram data.
\newblock {\em PloS one}, 9(4):e93365, 2014.

\bibitem{ramdas2017wasserstein}
A.~Ramdas, N.~G. Trillos, and M.~Cuturi.
\newblock On {W}asserstein two-sample testing and related families of
  nonparametric tests.
\newblock {\em Entropy}, 19(2):47, 2017.

\bibitem{reeves2007review}
J.~Reeves, J.~Chen, X.~L. Wang, R.~Lund, and Q.~Q. Lu.
\newblock A review and comparison of changepoint detection techniques for
  climate data.
\newblock {\em Journal of applied meteorology and climatology}, 46(6):900--915,
  2007.

\bibitem{rigollet2022sample}
P.~Rigollet and A.~J. Stromme.
\newblock On the sample complexity of entropic optimal transport, 2022.

\bibitem{rockafellar1970convex}
R.~T. Rockafellar.
\newblock {\em Convex analysis}, volume~18.
\newblock Princeton university press, 1970.

\bibitem{santambrogio2015optimal}
F.~Santambrogio.
\newblock Optimal transport for applied mathematicians.
\newblock {\em Birk{\"a}user, NY}, 55(58-63):94, 2015.

\bibitem{shirdhonkar2008approximate}
S.~Shirdhonkar and D.~W. Jacobs.
\newblock Approximate earth mover’s distance in linear time.
\newblock In {\em 2008 IEEE Conference on Computer Vision and Pattern
  Recognition}, pages 1--8. IEEE, 2008.

\bibitem{siegmund1995using}
D.~Siegmund and E.~Venkatraman.
\newblock Using the generalized likelihood ratio statistic for sequential
  detection of a change-point.
\newblock {\em The Annals of Statistics}, pages 255--271, 1995.

\bibitem{smirnov1939estimation}
N.~V. Smirnov.
\newblock On the estimation of the discrepancy between empirical curves of
  distribution for two independent samples.
\newblock {\em Bull. Math. Univ. Moscou}, 2(2):3--14, 1939.

\bibitem{sriperumbudur2012empirical}
B.~K. Sriperumbudur, K.~Fukumizu, A.~Gretton, B.~Sch{\"o}lkopf, and G.~R.
  Lanckriet.
\newblock On the empirical estimation of integral probability metrics.
\newblock {\em Electronic Journal of Statistics}, 6:1550--1599, 2012.

\bibitem{wang2013linear}
W.~Wang, D.~Slep{\v{c}}ev, S.~Basu, J.~A. Ozolek, and G.~K. Rohde.
\newblock A linear optimal transportation framework for quantifying and
  visualizing variations in sets of images.
\newblock {\em International journal of computer vision}, 101(2):254--269,
  2013.

\bibitem{wasserman2013all}
L.~Wasserman.
\newblock {\em All of statistics: a concise course in statistical inference}.
\newblock Springer Science \& Business Media, 2013.

\bibitem{yu2022comprehensive}
X.~Yu and Y.~Cheng.
\newblock A comprehensive review and comparison of cusum and
  change-point-analysis methods to detect test speededness.
\newblock {\em Multivariate Behavioral Research}, 57(1):112--133, 2022.

\end{thebibliography}

\appendix
\onecolumn

\section{Approaches to CPD}\label{sec:approach_cpd}

There are several flavors of CPD e.g.  supervised \cite{chang2019kernel} or unsupervised \cite{li2015scan}, online \cite{adams2007bayesian} or offline \cite{li2015scan}, number of change points (single or multiple), the dimension $d$ of the signal, and if the signal is parametric or nonparametric. Parametric approaches make specific assumptions about the underlying data distributions and detect change points based on statistics computed from pre-change and post-change distributions \cite{chamroukhi2013joint, lee2018time}. The most widely used parametric approaches are cumulative sum- \cite{yu2022comprehensive}, and generalized likelihood ratio (GLR)- \cite{siegmund1995using} based CPD.  These parametric approaches are mostly suited for {\it quickest} change point detection where the goal is to detect a change in the quickest time. In contrast, nonparametric methods are able to detect change points without any assumptions on the underlying distribution.

Typically nonparametric methods use the sliding-window technique (Algorithm \ref{alg:alg1}) to measure the similarity at every possible point of the signal via a statistical two-sample GoF testing framework. Classical and popular statistics such as  Kolmogorov-Smirnov \cite{kolmogorov1933sulla, smirnov1939estimation, massey1951kolmogorov} and Cram\'{e}r-von-Mises  \cite{cramer1928composition, mises2013probability, anderson1962distribution, hawkins2010nonparametric} statistics have been used for CPD. However, these statistics rely on comparing empirical CDF, and only apply when the data dimension $d=1$. Maximum mean discrepancy (MMD) \cite{gretton2012kernel} is a GoF statistic that comes from a family of integral probability metrics \cite{sriperumbudur2012empirical} and has been used to detect change points when $d>2$ \cite{li2015scan}. Recently, OT based statistics have also been proposed for sliding-window-based CPD for multivariate signal: Wasserstein-1 (W1) distance \cite{cheng2020on}, 
 a distribution-free variant of Wasserstein distance
that measures the Wasserstein distance of the Q-Q function
to the uniform measure known as Wasserstein-Quantile test
(WQT)\cite{cheng2020optimal, ramdas2017wasserstein}, and Sinkhorn divergence \cite{ahad2022learning}.  

\section{Proofs from Section \ref{sec:CPD_main}}

\subsection{Proof of Theorem \ref{thm:re_non_smooth}} \label{sec:re_non_smooth}

\begin{proof}[\textbf{Proof of Theorem \ref{thm:re_non_smooth}}]
    Without loss of generality we can assume that $\varepsilon < \sup_{Q_X,Q_Y} \RE(Q_X,Q_Y)$ since otherwise the claim holds by the positivity of $\RE$ (See Section \ref{sec:sre_to_re_non_asymp}). Now let
    Let $P_X',P_Y'$ be absolutely continuous and such that 
    \begin{equation*}
        \RE(P_X', P_Y') \geq \sup_{Q_X,Q_Y} \RE(Q_X,Q_Y) - \varepsilon.
    \end{equation*}
    Let $w = W_1(P_X',P_Y')$. Note that $\RE(P_X',P_Y') > 0$ implies  $w > 0$ so that that the map $S:\mathbb{R}^{d}\rightarrow \mathbb{R}^{d}$ given by 
    \begin{equation*}
        S(x) = \frac{\delta}{w}x,
    \end{equation*}
    is well-defined. Let $P_X = S\#P_X'$ and $P_Y = S\#P_Y'$, and note that $P_X,P_Y$ are also both absolutely continuous. We will show that $\RE(P_X,P_Y) = \RE(P_X',P_Y')$, which can be seen as a consequence of the fact that the optimal transport map has an invariance to scaling. Indeed let $\R'$ denote the rank map of $P_\lambda' = \lambda P_X' + (1-\lambda)P_Y'$ and let $\R$ denote the rank of  $P_\lambda = \lambda P_X + (1-\lambda)P_Y$. We claim that $\R = \R' \circ S^{-1}$. To see that $\R' \circ S^{-1}$ is a valid map, note that $S^{-1}\#P_\lambda = P_\lambda'$ and therefore
    \begin{equation*}
        (\R' \circ S^{-1}) \# P_\lambda = \R'\#(S^{-1}\#P_\lambda) = \R'\#P_\lambda' = \text{Unif}([0,1]^d).
    \end{equation*}
    To see that it is optimal, we can compute it's gradient as
    \begin{equation*}
        \nabla (\R' \circ S^{-1})(x) = \nabla S^{-1}(x) \nabla \R'(S^{-1}(x)) = \frac{w}{\delta} \mathbb{I} \nabla \R'(S^{-1}(x)) =  \frac{w}{\delta} \nabla \R'(S^{-1}(x)).
    \end{equation*}
    Since $\R'$ is the gradient of a convex function, $\nabla \R'(S^{-1}(x))$ is a positive semi-definite matrix and since $w/\delta > 0$ it must be that $(w/\delta) \nabla \R'(S^{-1}(x))$ is also positive semi-definite, which shows that $\R' \circ S^{-1}$ is the gradient of a convex function. Recalling that $P_X, P_Y$ are absolutely continuous and using Brenier's theorem, this shows that $\nabla \R'(S^{-1}(x))$ is the unique optimal map. This confirms $\R = \R' \circ S^{-1}$.
    
    In particular, this establishes
    \begin{align*}
        \RE(P_X, P_Y)^2 &= 2\mathbb{E}_{P_X,P_Y}\norm{\R(X) - \R(Y)} - \mathbb{E}_{P_X} \norm{\R(X) - \R(X')} - \mathbb{E}_{P_Y}\norm{\R(Y) - \R(Y')} \\
        &= 2\mathbb{E}_{P_X,P_Y}\norm{\R'(S^{-1}((X))) - \R'(S^{-1}((Y)))} - \mathbb{E}_{P_X}\norm{\R'(S^{-1}((X))) - \R'(S^{-1}((X')))} \\
        &\hspace{2cm} - \mathbb{E}_{P_Y}\norm{\R'(S^{-1}((Y))) - \R'(S^{-1}((Y')))} \\
        &= 2\mathbb{E}_{P_X'P_Y'}\norm{\R'(X) - \R'(Y)} - \mathbb{E}_{P_X'} \norm{\R'(X) - \R'(X')} - \mathbb{E}_{P_Y'} \norm{\R'(Y) - \R'(Y')} \\
        &= \RE(P_X', P_Y')^2
    \end{align*}
    Taking square roots and using the assumptions on $P_X'$ and $P_Y'$ shows
    \begin{equation*}
        \RE(P_X, P_Y) = \RE(P_X', P_Y') \geq \sup_{Q_X,Q_Y} \RE(Q_X,Q_Y) - \varepsilon.
    \end{equation*}
    To conclude, let $T'$ be the optimal map in terms of $W_1$ from $P_X'$ to $P_Y'$. Then we have 
    \begin{align*}
        W_1(P_X, P_Y) &\leq \int \norm{(\delta/w)T'((w/\delta)x) - x} dP_X(x) \\
        &= \frac{\delta}{w}\int \norm{T'((w/\delta)x) - (w/\delta)x} dP_X(x) \\
        &= \frac{\delta}{w} \int \norm{T'(x) - x} dP_X'(x) \\ 
        &= \frac{\delta}{w}W_1(P_X',P_Y') = \frac{\delta}{w}w = \delta
    \end{align*}
    where we have used the fact that $(\delta/w)T'((w/\delta))\#P_X = P_Y$, which can be verified in a similar way as above. This shows that the pair $P_X,P_Y$ satisfy the two required properties. 
\end{proof}
\subsection{Proof of Theorem \ref{thm:sre_smooth}} \label{sec:sre_smooth}

Before proving Theorem \ref{thm:sre_smooth} we first establish the Lipschitz continuity of the entropic map.
\begin{lemma} \label{lem:bounded_lip}
    Suppose that $\supp(Q) \subseteq B_2^d(u,r)$ for some $u \in \mathbb{R}^d$, $r>0$. Then the entropic transport map $\Teps$ from $P$ to $Q$  is $(4r^2/\varepsilon)$-Lipschitz continuous.
\end{lemma} 
For convenience we will introduce the notation $\Sigma_\varepsilon^x \triangleq \cov_{Y \sim \pi_\varepsilon^x} (Y)$. We first recall a known result in the literature.
\begin{lemma}[\cite{sinho2022contraction} Lemma 1] \label{lem:jacobian} Let $\pi_\varepsilon^x$ denote the conditional distribution of $\pi_\varepsilon$ given $X = x$. Then 
\begin{equation*}
    \nabla \Teps(x) = \frac{1}{\varepsilon} \cov_{Y \sim \pi_\varepsilon^x} (Y)=\Sigma_\varepsilon^x.
\end{equation*}
\end{lemma}
Using that the Lipschitz constant of a vector-valued function is the supremum of the operator norm of its Jacobian, we have the following corollary.
\begin{corollary} \label{cor:pre_lip}
    The entropic map is $L$-Lipschitz with respect to the Euclidean distance with
    \begin{equation*}
        L = \frac{1}{\varepsilon}\sup_{x \in \Omega} \norm{\Sigma_\varepsilon^x}_{\text{op}}.
    \end{equation*}
\end{corollary}
\begin{proof}[\textbf{Proof of Lemma \ref{lem:bounded_lip}}]
    Note that for all $x$ the support of $\pi_\varepsilon^x$ is contained in $B_2^d(u,r)$. Let $\Bar{Y} = \mathbb{E}_{Y \sim \pi_\varepsilon^x}[Y] \in B_2^d(0,r)$. Letting $Z = Y - \Bar{Y}$ we have by the translation invariance of the covariance matrix  and the fact that $Z$ is mean-zero 
    $$\Sigma^x_\varepsilon = \cov(Z) = \mathbb{E}ZZ^{\top}.$$
    Note that 
    \begin{equation*}
        Z \in B_2^d(u,r) - \Bar{Y} \subset (u + B_2^d(0,r)) - (u + B_2^d(0,r)) = B_2^d(0,2r)
    \end{equation*} and therefore for any unit  $v \in \mathbb{R}^d$ with $\norm{v} = 1$ we have
    \begin{align*}
        v^{\top}\Sigma^x_\varepsilon v = v^{\top}\mathbb{E}[ZZ^{\top}]v = \mathbb{E}[(v^{\top}Z)(Z^{\top}v)] \leq \mathbb{E}[(\norm{v}\cdot\norm{Z})(\norm{Z}\cdot\norm{v})] \leq \mathbb{E}\norm{Z}^2 \leq 4r^2.
    \end{align*}
    This implies that for all $x \in \Omega$ we have
    $\norm{\Sigma^x_\varepsilon}_{op} \leq 4r^2$. Taking the supremum over $x$ and applying Corollary \ref{cor:pre_lip} we have
    $L = \frac{1}{\varepsilon} \sup_{x \in \Omega} \norm{\Sigma^x_\varepsilon}_{op} \leq \frac{1}{\varepsilon}(4r^2)$
    which proves the result.
\end{proof}

\begin{proof}[\textbf{Proof of Theorem \ref{thm:sre_smooth}}]
    First note that we are using $Q = \text{Unif}([0,1]^d)$ and $\supp(\nu) \subset B_2^d((1/2)\bm{1}, \sqrt{d/4})$ where $\bm{1}$ denotes the all 1 vector in $\mathbb{R}^d$. Therefore by Lemma \ref{lem:bounded_lip} we have that soft rank map $\sR$ from $P_\lambda$ to $Q$ is $(d/\varepsilon)-$Lipschitz. 
    
    In addition let $T$ be a transport map from $P_X$ to $P_Y$ such that
    \begin{equation}
        \mathbb{E}_{X}\norm{T(X) - X} = W_1(P_X,P_Y).
    \end{equation}
    and let $\sR$ be the soft rank map.
    \begin{align*}
        &\sRE(P_X,P_Y)^2 \\
        &= 2\mathbb{E}_{X,Y}\left [ \norm{\Teps^\lambda(X) - \Teps^\lambda(Y)} \right ] - \mathbb{E}_{X,X'}\left [ \norm{\Teps^\lambda(X) - \Teps^\lambda(X')} \right ] - \mathbb{E}_{Y,Y'}\left [ \norm{\Teps^\lambda(Y) - \Teps^\lambda(Y')} \right ] \\
        &= 2\mathbb{E}_{X,X'}\left [ \norm{\Teps^\lambda(X) - \Teps^\lambda(T(X'))} \right ] - \mathbb{E}_{X,X'}\left [ \norm{\Teps^\lambda(X) - \Teps^\lambda(X')} \right ] - \mathbb{E}_{X,X'}\left [ \norm{\Teps^\lambda(T(X)) - \Teps^\lambda(T(X'))} \right ] \\
        &= \mathbb{E}_{X,X'} \left [2\norm{\Teps^\lambda(X) - \Teps^\lambda(T(X'))} - \norm{\Teps^\lambda(X) - \Teps^\lambda(X')} - \norm{\Teps^\lambda(T(X)) - \Teps^\lambda(T(X'))} \right ] \\
        &\leq  \mathbb{E}_{X,X'}  \left | \norm{\Teps^\lambda(X) - \Teps^\lambda(T(X'))} - \norm{\Teps^\lambda(X) - \Teps^\lambda(X')} \right | \\
        &\hspace{2cm}+ \mathbb{E}_{X,X'} \left | \norm{\Teps^\lambda(X) - \Teps^\lambda(T(X'))} - \norm{\Teps^\lambda(T(X)) - \Teps^\lambda(T(X'))} \right | \\
        &\leq  \mathbb{E}_{X,X'} \norm{(\Teps^\lambda(X) - \Teps^\lambda(T(X'))) - (\Teps^\lambda(X) - \Teps^\lambda(X'))} \\ 
        &\hspace{2cm} + \mathbb{E}_{X,X'} \norm{(\Teps^\lambda(X) - \Teps^\lambda(T(X'))) - (\Teps^\lambda(T(X)) - \Teps^\lambda(T(X')))}  \\
        &= \mathbb{E}_{X,X'} \left [ \norm{\Teps^\lambda(T(X')) - \Teps^\lambda(X')} + \norm{\Teps^\lambda(X) - \Teps^\lambda(T(X)} \right ] \\
        &= 2\mathbb{E}_X \left [ \norm{\Teps^\lambda(T(X)) - \Teps^\lambda(X)} \right ] \\
        &\leq  2\mathbb{E}_X\left [\frac{d}{\varepsilon}\norm{T(X) - X} \right ] = \frac{2d}{\varepsilon}W_1(P_X,P_Y). 
    \end{align*}
    On the third line we have used that since $T$ transports $P_X$ to $P_Y$ that $T(X') \sim P_Y$. In the sixth line we have used the reverse-triangle inequality. The ninth uses the fact that $\Teps^\lambda$ is $(d/\varepsilon)$-Lipschitz and the last line is by the assumption on $T$.
\end{proof}
\subsection{Proof of Theorem \ref{thm:sre_to_re_non_asymp}}
\label{sec:sre_to_re_non_asymp}

\begin{proof}[\textbf{Proof of Theorem \ref{thm:sre_to_re_non_asymp}}]
We begin by recalling that $\RE$ and $\sRE$ have an equivalent formulation \cite{masud2021multivariate}
\begin{align}
       \sRE(P_X,P_Y)^2 &= C_d \int_{\mathcal S^{d-1}}\int_{\mathbb R} \left (\mathbb P\big(a^\top \sR(X)\leq t\big) - \mathbb P\big(a^\top \sR(Y)\leq t\big)\right )^2 dt  d\kappa(a) \notag, \\
      {\RE}(P_X,P_Y)^2 &= C_d \int_{\mathcal S^{d-1}}\int_{\mathbb R} \left (\mathbb P\big(a^\top \R(X)\leq t\big) - \mathbb P\big(a^\top \R(Y)\leq t\big)\right )^2 dt  d\kappa(a) \notag
    \end{align}
where $C_d = \left (2\Gamma(d/2) \right )^{-1} \sqrt{\pi}(d-1)\Gamma\big((d-1)/2\big)$ is an appropriate normalizing constant.
Let 
\begin{align*}
    u_{a,t} &  \triangleq \mathbb{P}\big(a^\top \sR(X)\leq t\big) - \mathbb P\big(a^\top \sR(Y)\leq t\big), \\
    v_{a,t} & \triangleq \mathbb{P}\big(a^\top \R(X)\leq t\big) - \mathbb P\big(a^\top \R(Y)\leq t\big).
\end{align*}
Then, it follows that 
\begin{align*}
    | \sRE(P_X,P_Y)^2 - \RE(P_X,P_Y)|^2 & = C_d \left | \int_{\mathcal{S}^{d-1}}  \int_\mathbb{R} u_{a,t}^2 - v_{a,t}^2  dt d\kappa(a)  \right | \\
    & \leq  C_d \int_{\mathcal{S}^{d-1}}  \int_\mathbb{R} |u_{a,t}^2 - v_{a,t}^2|  dt d\kappa(a) \\
    & = C_d \int_{\mathcal{S}^{d-1}}  \int_\mathbb{R} |(u_{a,t} - v_{a,t})(u_{a,t} + v_{a,t})|  dt d\kappa(a)\\
    & \leq 2 C_d \int_{\mathcal{S}^{d-1}}  \int_\mathbb{R} |u_{a,t} - v_{a,t}|  dt d\kappa(a).
\end{align*}
We can further simplify the last integral as
\begin{align*}
    \int_{\mathbb{R}} |u_{a,t} - v_{a,t}| dt &= \int_{\mathbb{R}} \left | \mathbb{P}\left ( a^{\top}\sR(X) \leq t \right ) - \mathbb{P}\left ( a^{\top}\sR(Y) \leq t \right ) - \mathbb{P}\left ( a^{\top}\R(X) \leq t \right ) + \mathbb{P}\left ( a^{\top}\R(Y) \leq t \right ) \right | dt \\
    &\leq \int_{\mathbb{R}} \left | \mathbb{P}\left ( a^{\top}\sR(X) \leq t \right ) - \mathbb{P}\left ( a^{\top}\R(X) \leq t \right ) \right | dt + \int_{\mathbb{R}} \left | \mathbb{P}\left ( a^{\top}\sR(Y) \leq t \right ) - \mathbb{P}\left ( a^{\top}\R(Y) \leq t \right ) \right |.
\end{align*}
Let $X_{a}^{\varepsilon} = a^\top \sRE(X)$, $X_a = a^\top \RE(X)$ and $P_{X_a^\varepsilon}, P_{X_a}$ be their laws respectively. By the formula of Wasserstein-1 distance in dimension 1, 
\begin{align}
    \int_{\mathbb{R}} \left | \mathbb{P}\left ( a^{\top}\sR(X) \leq t \right ) - \mathbb{P}\left ( a^{\top}\R(X) \leq t \right ) \right | dt  &= \int_{\mathbb{R}} \left | \mathbb{P}\left ( X_a^\varepsilon \leq t \right ) - \mathbb{P}\left ( X_a \leq t \right ) \right | dt \notag \\
    &= W_1(P_{X_a^\varepsilon}, P_{X_a}) \notag  \\
    &= W_1((a^{\top}\sR)\#P_X, (a^{\top}\R)\#P_X) \notag  \\
    &\leq W_2((a^{\top}\sR)\#P_X, (a^{\top}\R)\#P_X) \notag  \\
    &\leq \sqrt{\mathbb{E}_X |a^{\top}\sR(X) - a^{\top}\R(X)|^2} \label{eq:tough_step} \\
    &\leq \sqrt{\mathbb{E}_X \norm{\sR(X) - \R(X)}^2} \notag 
\end{align}
In equation (\ref{eq:tough_step}) we have used the sub-optimal coupling $(a^{\top}\sR(\cdot) \otimes  a^{\top}\R(\cdot))\#P_X$. We have also used the fact that $W_1 \leq W_2$, Cauchy-Schwartz and the fact that $\|a\| = 1$. By an analogous computation we also have
\begin{align*}
    \int_{\mathbb{R}} \left | \mathbb{P}\left ( a^{\top}\sR(Y) \leq t \right ) - \mathbb{P}\left ( a^{\top}\R(Y) \leq t \right ) \right | dt \leq \sqrt{\mathbb{E}_Y \norm{\sR(Y) - \R(Y)}^2}. 
\end{align*}
Now under the assumption that $\R$ is $L$ Lipschitz and that $P_{\lambda}$ is supported on a bounded domain, we note from [\cite{carlier2022convergence}, Proposition 4.5] the following bound
\begin{align*}
    \| \sR - \R\|_{L^2(P_\lambda)}^{2} \leq L \varepsilon \log(1/\varepsilon) + O(\varepsilon) \triangleq g(\varepsilon).
\end{align*}
Now note that 
\begin{align*}
     \| \sR - \R\|_{L^2(P_\lambda)}^{2} = \lambda  \| \sR(X) - \R(X)\|_{L^2(P_X)}^{2} + (1 - \lambda) \| \sR(Y) - \R(Y)\|_{L^2(P_Y)}^{2} \leq g(\varepsilon).
\end{align*}
This implies both
\begin{align*}
    \| \sR(X) - \R(X)\|_{L^2(P_X)}^{2} &\leq \frac{1}{\lambda} g(\varepsilon) \\
    \| \sR(Y) - \R(Y)\|_{L^2(P_Y)}^{2} &\leq \frac{1}{1- \lambda}g(\varepsilon).
\end{align*}
Collecting the computations above we have, 
\begin{align*}
    | \sRE(P_X,P_Y)^2 - \RE(P_X,P_Y)|^2 & \leq  2C_d \int_{\mathcal{S}^{d-1}} \sqrt{\mathbb{E}_X \norm{\sR(X) - \R(X)}^2} + \sqrt{\mathbb{E}_Y \norm{\sR(Y) - \R(Y)}^2} d\kappa(a) \\
    &= 2C_d \gamma_d \sqrt{\mathbb{E}_X \norm{\sR(X) - \R(X)}^2} + 2C_d \gamma_d \mathbb{E} \sqrt{\mathbb{E}_Y \norm{\sRn(Y) - \R(Y)}^2} \\
    &\leq 2C_d\gamma_d \left (\frac{1}{\sqrt{\lambda}} + \frac{1}{\sqrt{1-\lambda}} \right )\sqrt{g(\varepsilon)},
\end{align*}
where $\gamma_d$ is the surface area of the unit sphere in $\mathbb{R}^d$.
\end{proof}

\subsection{Asymptotic Convergence of \texorpdfstring{$\sRE$}{sRE} to \texorpdfstring{$\RE$}{RE}} \label{sec:sre_to_re}

Before proceeding to the convergence of the $\sRE$ to $\RE$ as $\varepsilon \rightarrow 0$ we introduce the following technical lemma. 
\begin{proposition} \label{prop:erot} [Proposition 3.2, \cite{bernton2021entropic}] 
    Under the set-up of Theorem \ref{thm:erausquin_figali} and under the additional assumptions that OT cost is finite with a unique optimizer $T$, $\pi_\varepsilon \rightharpoonup \pi = [\text{Id} \otimes T]\#P$ as $\varepsilon \rightarrow 0$.
\end{proposition}
This result nearly implies convergence of $\Teps$ to $T$, since $\Teps$ is the conditional mean of $\pi_\varepsilon$ given fixed $x$. Indeed when paired with one additional trick one can show the following result. 
\begin{theorem} \label{thm:sre_to_re}
    Suppose that $P_\lambda = \lambda P_X + (1 - \lambda)P_Y$ for $\lambda \in (0,1)$  satisfies the assumptions of Theorem \ref{thm:erausquin_figali} and Proposition \ref{prop:erot}. Then 
    $$ \lim_{\varepsilon\rightarrow 0^+} \sRE(P_X,P_Y) = \RE(P_X,P_Y).$$
\end{theorem}
\begin{proof}
Since $\sRE$ and $\RE$ are a bounded continuous function of their corresponding rank maps, we claim it is sufficient to show that $\displaystyle\lim_{\varepsilon \rightarrow 0^+} \| \sR - \R\|_{L_2(P)} = 0.$  Let $\pi_\varepsilon$ be the entropic plan between $P$ and $Q$, and let $\pi_{\varepsilon}^{x}$ denotes the disintegration of the plan with respect to a fixed $x$. Then we have following \cite{carlier2022convergence}, 
\begin{align*}
    \mathbb{E}_{\pi_{\varepsilon}} \| Y - \R(X)\|^2 & = \int \int \| y - \R(x)\|^2 \pi_{\varepsilon}^{x}(y) d P_\lambda(x) \\
    & \geq \int\left \| \int ( y - \R(x)) d\pi_{\varepsilon}^{x}(y)\right\|^2 dP_\lambda(x) \\
    & = \int \| \sR (x) - \R(x) \|^2 d P_\lambda(x) \\
    & = \| \sR(x) - \R(x)\|_{L_2(P_\lambda)}^{2}.
\end{align*}
Note that under the assumptions, $\| y - 
\R(x) \|^2$ is a continuous bounded function, boundedness follows from the boundedness of the unit cube which in this case is the target measure and continuity follows from Theorem \ref{thm:erausquin_figali}. Therefore taking the limits with respect to $\varepsilon$ the result follows from Proposition \ref{prop:erot}. 
\end{proof}

\subsection{Proof of Theorem \ref{thm:sren_to_sre}} \label{sec:sren_to_sre}

In order to state the proof of this result, we must introduce an additional piece of notation
\begin{equation*}
    \sREnS(P_X,P_Y)^2 \triangleq \frac{1}{n^2} \sum_{i,j=1}^n 2\norm{\sR(X_i) - \sR(Y_j)} - \norm{\sR(X_i) - \sR(X_j)} - \norm{\sR(Y_i) - \sR(Y_j)}.
\end{equation*}
This is a mixture of both $\sRE(P_X,P_Y)^2$ and $\sREn(P_X,P_Y)^2$ since it takes the a finite sum when computing the integral and uses the population soft rank map. This one point of difference between both $\sRE(P_X,P_Y)^2$ (using the finite sum) and $\sREn(P_X,P_Y)^2$ (using the population map) makes it a natural intermediate step between the two terms. The choice of superscript $S$ is to indicate that it is a summation version of the sample soft rank energy. 
\begin{proof}[\textbf{Proof of Theorem \ref{thm:sren_to_sre}}]
    Adding and subtracting $\sREnS(P_X,P_Y)^2$ and using the triangle inequality we have
    \begin{align*}
        &\mathbb{E}\left |\sREn(P_X,P_Y)^2 - \sRE(P_X,P_Y)^2 \right | \\
        &\leq \mathbb{E}\left | \sREn(P_X,P_Y)^2 - \sREnS(P_X,P_Y)^2 \right | + \mathbb{E} \left | \sREnS(P_X,P_Y)^2 - \sRE(P_X,P_Y)^2 \right | \\
        &\leq \frac{24r\sqrt{1+\varepsilon^2}}{\sqrt{2n}}\exp(22r^2/\varepsilon) + 8 \sqrt{\frac{d\pi}{n}}
    \end{align*}
    where the last inequality applies Lemmas  \ref{lem:map_conv} and \ref{lem:mc_int}. 
\end{proof}

The proof of Theorem \ref{thm:sren_to_sre} requires two technical lemmas involving $\sREnS(P_X,P_Y)^2$. The first handles error incurred by the map estimate.
\begin{lemma} \label{lem:map_conv}
    With the notation defined above it holds that 
    \begin{equation*}
        \mathbb{E} \left |\sREn(P_X,P_Y)^2 - \sREnS(P_X,P_Y)^2 \right | \leq \frac{24r\sqrt{1+\varepsilon^2}}{\sqrt{2n}}\exp(22r^2/\varepsilon)
    \end{equation*}
\end{lemma}
\begin{proof}
    Through several applications of the triangle and reverse triangle inequalities one can show the first line of the following chain. The rest is using that $L^1 \leq L^2$ followed by the bound in Lemma \ref{lem:map_estimate}.
    \begin{align*}
        \left |\sREn(P_X,P_Y)^2  - \sREnS(P_X,P_Y)^2 \right | &\leq \frac{4}{n} \sum_{i=1}^n \norm{\sRn(X_i) - \sR(X_i)} + \norm{\sRn(Y_i) - \sR(Y_i)} \\
        &= 8\norm{\sRn - \sR}_{L^1((P_X^n + P_Y^n)/2} \\
        &\leq 8\norm{\sRn - \sR}_{L^2((P_X^n + P_Y^n)/2} \\
        &= 8\sqrt{\norm{\sRn - \sR}^2_{L^2((P_X^n + P_Y^n)/2}} \\
        &\leq \frac{24r\sqrt{1+\varepsilon^2}}{\sqrt{2n}}\exp(22r^2/\varepsilon)
    \end{align*}
    The last line is an application of Lemma \ref{lem:map_estimate}.
\end{proof}
The second lemma handles the error incurred by using a discrete sum instead of an integration.
\begin{lemma} \label{lem:mc_int}
    Let 
    \begin{equation*}
        \mathbb{E} \left |\sREnS(P_X,P_Y)^2 - \sRE(P_X,P_Y)^2 \right | \leq 8 \sqrt{\frac{d\pi}{n}}
    \end{equation*}
\end{lemma}
\begin{proof}
    In the setting of Lemma \ref{lem:h_sum_bound} let $h(x,y) = \norm{\sR(x) - \sR(y)}$. Then $\norm{h}_\infty \leq \sqrt{d}$. It can also be seen from the definitions that
    \begin{align*}
        \sREnS(P_X,P_Y)^2 &= \frac{2}{n^2}\sum_{i,j=1}^n h(X_i,Y_j) - \frac{1}{n^2}\sum_{i,j=1}^n h(X_i,X_j) - \frac{1}{n^2}\sum_{i,j=1}^n h(Y_i,Y_j) \\
        \mathbb{E}\sREnS(P_X,P_Y)^2 &= \mathbb{E}[2h(X,Y) - h(X,X') - h(Y,Y')] \\
        &= \sRE(P_X,P_Y)^2.
    \end{align*}
    Therefore by Lemma \ref{lem:h_sum_bound} we have
    \begin{align}
        \mathbb{E} \left | \sREnS(P_X,P_Y)^2 - \sRE(P_X,P_Y)^2 \right | \leq 8\sqrt{\frac{d\pi}{n}}.
    \end{align}
\end{proof}

\section{Datasets and Hyperparameters} \label{sec:datasets}

In Section \ref{sec:experiments} we consider a class of synthetic data, as well as 5 real data sets, described below.
\begin{table}[t]
    \centering
    \scriptsize
    \setlength{\tabcolsep}{1pt}
    \begin{tabular}{|c|c|c|c|c|c|c|c|c|c|c|}
            \hline
         &seg:1& seg:2 &seg:3& seg:4& seg:5& seg:6 & seg:7& seg:8 & seg:9& seg:10  \\ \hline
         Distribution&$\mathcal N(0_d, .001 I_d)$& $\mathcal N(0_d, .01 I_d)$ & $\mathcal N(1_d,  I_d)$&$\text{Laplace}(0_d, I_d)$&$\mathcal N(1_d, I_d)$ &$\Gamma(2, 2)$& $\mathcal N(0_d, .1 I_d)$ & $\mathcal N(1_d, \Sigma)$ & $\mathcal N(0_d, .01I_d)$& $\mathcal N(0_d, .001I_d)$ \\ \hline
         Length & 300 & 400 & 500 & 300&400 & 300 & 200 &300 & 200 & 400 \\\hline
    \end{tabular}
    \caption{Underlying distribution and length of each segment of the synthetic dataset.}
    \label{tab:toy}
\end{table}
\begin{itemize}
    \item[(a)] \underline{Synthetic data:} To generate a synthetic dataset $(X_t)\subset \mathbb R^d$, $d= 10$, we consider several distributions via concatenating distinct time segments of different lengths, where the samples from different segments are drawn from different distributions as described in Table \ref{tab:toy}.
    This dataset has a length of 3300 samples and 9 change points. For this dataset, we vary the window size $n$ to assess the effect of window size on the performance. We choose minimum horizontal distance $\Delta= n$. For every $n$, we use detection margin $\xi = 20$. Additional results on the synthetic dataset are given in Figure \ref{fig:toy_dataset_2col} and Table \ref{tab:synthetic_result}. 
    \item[(b)] \underline{HASC-PAC2016}: A human activity recognition dataset consists of over 700 three-axis accelerometer sequences sampled at 100Hz  where the subjects perform six different actions, `stay', `walk', `jog', `skip', `stairs up', and `stairs down' ($d=3$). Time points that exhibits changes in activity are annotated as ground truth. To evaluate the performance of the CPD methods on this dataset, we consider the 20 longest sequences which having an average length of 17,000 samples and 15 change points. We choose $n= 500, \xi = 200, \Delta = 250$ for this dataset. We use $\varepsilon= 0.1$ to compute sRE. 
    \item[(c)] \underline{HASC-2011}: Another human activity recognition dataset where people perform six different actions, `stay', `walk', `escalator up', `elevator up', `stairs up' and `stairs down' and an accelerometer takes three-dimensional data ($d=3$). Change points are annotated in the same way as in HASC-PAC2016 dataset. This dataset consists of 2 sequences, which have an average length of 37000 samples and 46 change points. We select $n= 500, \xi = 200, \Delta = 250$ and $\varepsilon = 0.1$ to compute sRE for this dataset.
    
    \item[(d)] \underline{Bee dance}: A dataset containing 6 three-dimensional sequences each collected from the movement of dancing honeybees communicating through three actions: `turn right', `turn left' and `waggle' ($d=3$). The first two dimensions correspond to the spatial $x$ and $y$ coordinates and the third one is the heading angle of the bees captured via video tracking. Each sequence has an average length of 790 samples and 19 change points. For this dataset, we choose $n= 20, \xi = 10, \Delta = 10$ and $\varepsilon= 1$ for sRE.
    
    \item [(e)] \underline{Salinas A Hyperspectral image:} A high-dimensional image consisting of $83\times 86$ pixels, each a $d=224$ dimensional vector of spectral reflectances.  The data was recorded by the Airborne Visible/Infrared Imaging Spectrometer (AVIRIS) sensor over farmland in Salinas Valley, California, USA in 1998 at a spatial resolution of 1.3 m. Spectral signatures, ranging in recorded wavelength from 380 nm to 2500 nm across 224 spectral bands, were recorded.  For this data, each pixel is treated as a sample. Samples are divided into six different classes, corresponding to the material classes of the pixels (e.g., broccolli greens, corn greens, lettuce of different ages).  While not a time series, we can read the pixels row-wise and annotate a pixel as a change point if it has a different labeled class from the previous pixel. For this paper, we consider first $500$ samples which contain 54 change points. We choose $n = 10, \xi= 2, \Delta= 2$. To compute sRE, we use $\varepsilon=1$.  
        
    \item[(f)] \underline{ECG:} A one dimensional ($d=1$) dataset consisting of a single sequence having a length of $8600$ samples and $89$ change points, where each change point represent an abnormal heartbeat. For this dataset, we choose $n= 50, \xi= 20, \Delta = 25$ and $\varepsilon = 0.1$ to compute sRE.
\end{itemize}

\subsection{Descriptions of Hyperparameters}

In the evaluation we require several hyperparamters which have been mentioned in the main text. The complete list of the hyperparameters as well as their effects are summarized below
\begin{itemize}
    \item \underline{Window Size $n$:} The number of samples seen in each one of the windows. Increasing $n$  typically leads to smoother changes in the GoF statistic over time and a gain in performance. However if the window size is too large multiple change points may fall within the same half of a window which may be problematic.
    \item \underline{Threshold $\eta$:} The minimum value the GoF statistic must take in order to be registered as a change point. Since the GoF statistic computed on samples is very rarely a constant 0 it is useful to choose a small threshold below which no change points can be predicted. This prevents the peak finding procedure from proposing change points in regions where there are clearly no changes. Increasing $\eta$ leads to discarding more and more proposed changed points. However a value of $\eta$ which is too large may lead to missing subtle change points.
    \item \underline{Horizontal Displacement $\Delta$:} The minimum distance apart which two predicted change points must be, that is $|\hat{\tau}_j - \hat{\tau}_k| \geq \Delta$ for every $j,k$. Using this prevents the prediction of several change points in rapid succession due to small sub-peaks near a single large peak. The larger the setting of $\Delta$ is taken the more spaced out the predicted change points must be. Taking $\Delta$ too large relative to the frequency of the true change points may be problematic as it can force true change points to be ignored because of the horizontal displacement constraint.
    \item \underline{Margin of Error $\xi$:} The maximum allowable distance which a predicted change point $\hat{\tau}_k$ can be from a true change point $\tau_j$ while still being considered correct. If $|\hat{\tau}_k - \tau_j| \leq \xi$ than it is considered to have correctly identified $\tau_j$. This only impacts the numerical evaluation of the methods and scores will increase as the margin of error increases. The choice of $\xi$ should depend on the quality of the annotated change points which can be noisy. A large $\xi$ may inflate the performance of the methods. A $\xi$ which is too small may lead to poor scores for methods which perform well but do not consistently place the precise change points in a small target, especially if there is ambiguity in the proper placement of the annotations,
\end{itemize}
\subsection{Results on Synthetic data}
\begin{figure}[t]
    \centering
        \includegraphics[width =\linewidth]{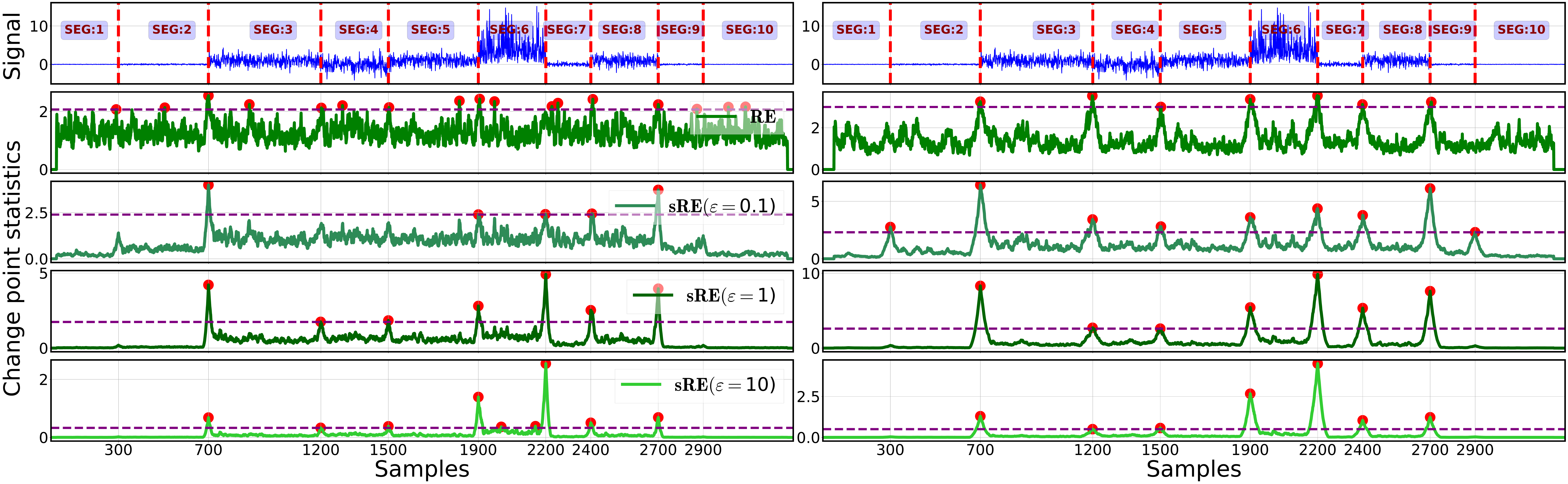}
    \vspace{-5mm}
    \caption{
    Change point statistics using RE and sRE on synthetic dataset with threshold $\eta$ (horizontal dashed purple) providing the best F1-score, the detected change points (red dot), and true labels (top row, vertical dashed red). Both RE and sRE statistics become smoother as the window size $n$ increases. Additionally, for any $n$, the sRE statistics become smoother as the value of $\varepsilon$ increases.}
    \label{fig:toy_dataset_2col}
\end{figure}
\begin{table*}[ht]
\small
\setlength{\tabcolsep}{5pt}
\caption{Average AUC-PR and best F1-scores  on synthetic dataset (taken over 25 independent instances)  w.r.t. window size $n$ (bold: best, italic: second best). As $n$ increases, each method gains a significant improvement in both metrics.}
\center
  \begin{tabular}{c|c|c|c|c|c|c|c|c}
    \hline
    \multirow{2}{*}{Method} &
      \multicolumn{4}{c|}{AUC-PR} &
      \multicolumn{4}{c}{Best F1-score} \\ \cline{2-9}
    & $n= 25$ & $n = 50$ & $n= 100$ & $n= 200$ & $n= 25$ & $n = 50$ & $n= 100$ & $n= 200$ \\
    \hline \hline
    M-stat \cite{li2015scan}                    &0.442 & 0.665 & 0.879 & 0.879 & 0.625 & 0.737 & {\bf 1.0} & {\bf 1.0}\\
    Sinkdiv \cite{ahad2022learning}                  & 0.329 & 0.425 & 0.598 & 0.657 & 0.510 & 0.695 & 0.834 & 0.875 \\
    W1 \cite{cheng2020optimal}                        &  0.259 & 0.462 & 0.563 & 0.756 & 0.459 & 0.747 & 0.825 & 0.889\\
    WQT \cite{cheng2020on}                       &{\bf 0.867} & 0.879 & {\bf 0.887} & 0.882 & {\bf 0.947} & {\bf 1.0} & {\bf 1.0} & {\bf 1.0}\\
    
    RE                        &0.377 & 0.717 & 0.746 & 0.767 & 0.538 & {\it 0.875} & {\it 0.875} & 0.875\\
    sRE ($\varepsilon = 0.1$)    &0.631 & {\bf 0.882} & 0.885 & {\bf 0.886} & 0.724 & {\bf 1.0} & {\bf 1.0} & {\bf 1.0} \\
    sRE ($\varepsilon = 1$)      &{\it 0.684} & {\it 0.734} & {\it 0.782} & {\it 0.771} & {\it 0.875} & {\it 0.875} & {\it 0.875} & {\it 0.941}\\
    sRE ($\varepsilon = 10$)     &0.656 & 0.724 & 0.778 &{\it 0.776} & 0.778 &  {\it 0.875} & {\it 0.875} & {\it 0.941}\\
    \hline
  \end{tabular}
  \label{tab:synthetic_result}
\end{table*}
\newpage
\section{Background on Entropic Optimal Transport} \label{sec:bacground_on_eot}
\subsection{Dual Optimality Conditions}
In this section of the appendix we review the essentials of entropic optimal transport, mainly following the work in \cite{rigollet2022sample}. The most important fact is the duality form of the problem.
\begin{theorem} Let $P,Q$ be distributions on $\mathbb{R}^d$ with bounded support and let $\varepsilon > 0$. Then 
\begin{equation*}
    S_\varepsilon(P,Q) = \sup_{(f,g) \in L^\infty(P) \times L^\infty(Q)} \int f dP + \int g dQ - \varepsilon \int\int \exp \left ( \frac{1}{\varepsilon}\left [ f(x) + g(y) - \frac{1}{2}\norm{x-y}^2 \right ] \right ) dP(x)dQ(y) + \varepsilon.
\end{equation*}
The supremum is attained at a pair $(f_0,g_0) \in L^\infty(P) \times L^\infty(Q)$ of dual potentials, which are unique up to the translation $(f_0,g_0) \mapsto (f_0 + c, g_0 - c)$ for $c \in \mathbb{R}$.

Moreover, primal and dual solutions are linked via the following relationships. For any pair $(f,g) \in L^\infty(P) \times L^\infty(Q)$, let $\pi$ be the measure with density
\begin{equation} \label{eq:opt_density}
    \frac{d\pi}{d(P \otimes Q)}(x,y) = \exp \left ( \frac{1}{\varepsilon}\left [ f(x) + g(y) - \frac{1}{2}\norm{x-y}^2 \right ] \right ).
\end{equation}
Then the pair $(f,g)$ is optimal for $S_\varepsilon(P,Q)$ if and only if $\pi$ is a coupling of $P$ and $Q$ and $\pi$ is optimal for $S_\varepsilon(P,Q)$.
\end{theorem}
For proof and discussion of this result see \cite{marino2020optimal}.  We will let $(f_*,g_*)$ denote the optimal dual potentials in $S_\varepsilon(P,Q)$ which satisfy $\int g_* dQ = 0$. Given a set of samples $Y_1,...,Y_n \sim Q$ We define another pair of optimal dual potentials $(\overline{f}_*, \overline{g}_*)$ by
\begin{equation*}
    \overline{f}_* \triangleq f_* + \frac{1}{n} \sum_{i=1}^{n} g_*(Y_i), \hspace{1cm} \overline{g}_* \triangleq g_* - \frac{1}{n} \sum_{i=1}^{n} g_*(Y_i).
\end{equation*}
In addition, let $(f_n,g_n)$ be the unique optimizers for $S_\varepsilon(P^n,Q^n)$ such that $\frac{1}{n}\sum_{i=1}^n g_n(Y_i) = 0$. Using these definitions, the optimality condition (\ref{eq:opt_density}) for $(f_*,g_*)$ and can be re-written as
\begin{align*}
    f_*(x) &= -\varepsilon \ln \left ( \int \exp \left ( \varepsilon \left [g_*(y) - \frac{1}{2}\norm{x-y}^2 \right ] \right ) dQ(y) \right ), \\
    g_*(y) &= -\varepsilon \ln \left ( \int \exp \left ( \varepsilon \left [f_*(x) - \frac{1}{2}\norm{x-y}^2 \right ] \right ) dP(x) \right ),
\end{align*}
which holds for $P$ a.e. $x$ and $Q$ a.e. $y$. 
Similarly for $(f_n,g_n)$ the optimality condition can be restated as
\begin{align*}
    f_n(x) &= -\varepsilon \ln \left ( \int \exp \left ( \frac{1}{\varepsilon} \left [g_n(y) - \frac{1}{2}\norm{x-y}^2 \right ] \right ) dQ^n(y) \right ), \\
    g_n(y) &= -\varepsilon \ln \left ( \int \exp \left ( \frac{1}{\varepsilon} \left [f_n(x) - \frac{1}{2}\norm{x-y}^2 \right ] \right ) dP^n(x) \right ),
\end{align*}
which must hold for every $x \in \{X_1,...,X_n\}$ and every $y \in \{Y_1,...,Y_n\}$. Finally we will define the optimal relative densities as
\begin{align*}
    p_*(x,y) &\triangleq \frac{d\pi}{d(P \otimes Q)}(x,y) = \exp \left ( \frac{1}{\varepsilon}\left [ f_*(x) + g_*(y) - \frac{1}{2}\norm{x-y}^2 \right ] \right ), \\
    p_n(x,y) &\triangleq \frac{d\pi_n}{d(P^n \otimes Q^n)}(x,y) = \exp \left ( \frac{1}{\varepsilon}\left [ f_n(x) + g_n(y) - \frac{1}{2}\norm{x-y}^2 \right ] \right ),
\end{align*}
where the latter is only defined on the support of $P^n \otimes Q^n$. From these equations we can express the population and sample entropic maps as
\begin{align*}
    \Teps(x) &= \int yp_*(x,y) dQ(y), \\
    \Tepsn(x) &= \frac{1}{n} \sum_{i=1}^n Y_ip_n(x,Y_i),
\end{align*}
where the latter is only defined for $x \in \{X_1,...,X_n\}$.

\subsection{Dual Results}

Following \cite{rigollet2022sample} we denote the objective in the dual problem by
\begin{equation*}
    \Phi(f,g) \triangleq \int f dP + \int g dQ - \varepsilon \int\int \exp \left ( \frac{1}{\varepsilon}\left [ f(x) + g(y) - \frac{1}{2}\norm{x-y}^2 \right ] \right ) dP(x)dQ(y) + \varepsilon.
\end{equation*}
Furthermore, denote the empirical dual objective $\Phi_n$ by
\begin{equation*}
    \Phi_n(f,g) \triangleq \frac{1}{n} \sum_{i=1}^n \left [ f(X_i) + g(Y_i) \right ] - \frac{\varepsilon}{n^2}\sum_{i,j=1}^n \exp \left ( \frac{1}{\varepsilon} \left [f(X_i) + g(Y_j) - \frac{1}{2}\norm{X_i - Y_j}^2 \right ] \right ) + \varepsilon.
\end{equation*}

According to equation (2.7) in \cite{rigollet2022sample} one has
\begin{equation} \label{eq:norm_of_grad}
    \norm{\nabla \Phi_n(f,g)}_{L^2(P^n) \times L^2(Q^n)}^2 = \frac{1}{n} \sum_{i=1}^n \left (1 - \frac{1}{n}\sum_{j=1}^n p(X_i,Y_j) \right )^2 + \frac{1}{n} \sum_{j=1}^n \left ( 1 - \frac{1}{n}\sum_{i=1}^n p(X_i,Y_j) \right )^2
\end{equation}
where 
\begin{equation*}
    p(x,y) = \exp \left ( \frac{1}{\varepsilon} \left [ f(x) + g(y) - \frac{1}{2}\norm{x - y}^2 \right ] \right ).
\end{equation*}
Note that by the translational invariance of the dual potentials we have for every $c\in\mathbb{R}$ that
\begin{equation*}
    \Phi(f,g) = \Phi(f+c,g-c), \hspace{0.5cm} \Phi_n(f,g) = \Phi_n(f+c,g-c), \hspace{0.5cm} \nabla\Phi_n(f,g) = \nabla\Phi_n(f+c,g-c).
\end{equation*}

We now move onto a first basic structural result bounding the optimal dual potentials above and below. This result is essentially contained in both \cite{mena2019statistical,rigollet2022sample} and we include it here only to make the constants explicit in our case.
\begin{lemma} \label{lem:bounded_potentials}
    Let $P,Q \in \mathcal{P}(B(0,r))$ and let $(f_*,g_*), (f_n,g_n)$ be the optimal dual potentials as above for $S_\varepsilon(P,Q)$ and $S_\varepsilon(P^n, Q^n)$ respectively. Then
    \begin{equation*}
        \norm{f_n}_{L^\infty(P^n)}, \norm{g_n}_{L^\infty(Q^n)} \leq 2r^2, \hspace{1cm} \norm{f_*}_{L^\infty(P)}, \norm{g_*}_{L^\infty(Q)} \leq 2r^2.
    \end{equation*}
    In particular for $(P \otimes Q)$-a.e. $(x,y)$ and every $(x,y) \in \text{Supp}(P^n \otimes Q^n)$,
    \begin{align*}
        \exp \left ( -\frac{6r^2}{\varepsilon} \right ) &\leq p_*(x,y) \leq \exp \left ( \frac{4r^2}{\varepsilon} \right ), \\
        \exp \left ( -\frac{6r^2}{\varepsilon} \right ) &\leq p_n(x,y) \leq \exp \left ( \frac{4r^2}{\varepsilon} \right ).
    \end{align*}
\end{lemma}
\begin{proof}
    The optimality condition (\ref{eq:opt_density}) and marginal constraints on $\pi_*$ imply for $P$-a.e. $x$ that 
    \begin{align*}
        1 &= \int \exp \left ( \frac{1}{\varepsilon}\left[f_*(x) + g_*(y) - \frac{1}{2}\norm{x-y}^2 \right ]\right ) dQ(y) \\
        &\geq \exp\left ( \frac{1}{\varepsilon}\left [ f_*(x) - \frac{1}{2}(2r)^2 \right ] \right ) \int \exp \left ( \frac{1}{\varepsilon} g_*(y) \right ) dQ(y) \\
        &\geq \exp\left ( \frac{1}{\varepsilon}\left [ f_*(x) - 2r^2 \right ] \right )
    \end{align*}
    where the first inequality uses the fact that $x,y \in B(0,r)$ which implies $\norm{x-y} \leq 2r$ and the last inequality uses Jensen's inequality and the assumption that $\int g_* dQ = 0$. Taking logs on both sides and rearranging we see that for $P$ a.e. $x$
    \begin{equation}
        f_*(x) \leq 2r^2.
    \end{equation}
    Using this with the optimality and the marginal constraints on $\pi_*$ we have
    \begin{align*}
        1 &= \int \exp \left ( \frac{1}{\varepsilon}\left[f_*(x) + g_*(y) - \frac{1}{2}\norm{x-y}^2 \right ]\right ) dP(x) \\
        &\leq \int \exp \left ( \frac{1}{\varepsilon}\left[2r^2 + g_*(y) \right ]\right ) dP(x) \\
        &= \exp \left ( \frac{1}{\varepsilon}\left[2r^2 + g_*(y) \right ]\right ).
    \end{align*}
    Taking logarithms on both sides we have
    \begin{equation*}
        -2r^2 \leq g_*(y).
    \end{equation*}
    We next claim that $\int f_*(x) dP(x) \geq 0$ which can be seen from the fact that
    \begin{align*}
        0 &\leq S_\varepsilon(P,Q) \\
        &= \int f_* dP + \int g_* dQ - \varepsilon \int\int \exp \left ( \frac{1}{\varepsilon}\left [ f_*(x) + g_*(y) - \frac{1}{2}\norm{x-y}^2 \right ] \right ) dP(x)dQ(y) + \varepsilon \\
        &= \int f_* dP + 0 - \varepsilon \int d\pi + \varepsilon \\
        &= \int f_* dP
    \end{align*}
    where we have used the optimality condition on $g_*$ to remove one integral and (\ref{eq:opt_density}) to simplify the other. With this established we can repeat the proof above swapping the roles of $(x,f_*,P)$ with $(y,g_*,Q)$ respectively with the only important note being that 
    \begin{equation*}
        \int \exp \left ( \frac{1}{\varepsilon} f_*\right ) dP \geq \exp \left ( \frac{1}{\varepsilon} \int f_* dP \right ) \geq \exp \left ( \frac{1}{\varepsilon} \cdot 0 \right ) = 1
    \end{equation*}
    which is enough to show
    \begin{equation*}
        \exp\left ( \frac{1}{\varepsilon}\left [ g_*(y) - \frac{1}{2}(2r)^2 \right ] \right ) \int \exp \left ( \frac{1}{\varepsilon} f_*(x) \right ) dP(x) 
        \geq \exp\left ( \frac{1}{\varepsilon}\left [ g_*(y) - 2r^2 \right ] \right )
    \end{equation*}
    which mirrors the calculation above.

    The proof for $f_n$ and $g_n$ is completely analogous, just replacing the integrals with summations as needed. 

    The bounds on $p_*$ and $p_n$ follow from
    \begin{align*}
        p_*(x,y) &= \exp \left ( \frac{1}{\varepsilon} \left [ f_*(x) + g_*(y) - \frac{1}{2}\norm{x - y}^2 \right ] \right ) \\
        &\leq \exp \left ( \frac{1}{\varepsilon} \left [ 2r^2 + 2r^2 - 0 \right ] \right )\\ &= \exp \left ( \frac{4r^2}{\varepsilon} \right ), \\
        p_*(x,y) &= \exp \left ( \frac{1}{\varepsilon} \left [ f_*(x) + g_*(y) - \frac{1}{2}\norm{x - y}^2 \right ] \right ) \\
        &\geq \exp \left ( \frac{1}{\varepsilon} \left [ -2r^2 - 2r^2 - 2r^2 \right ] \right ) \\
        &= \exp \left ( -\frac{6r^2}{\varepsilon} \right ),
    \end{align*}
    and an analogous calculation for $p_n$.
\end{proof}
For convenience we introduce the set 
\begin{equation*}
    \mathcal{S}_L \triangleq \left \{ (f,g) \in L^\infty(P^n) \times L^\infty(Q^n) : \norm{f}_{L^\infty(P^n)},\norm{g}_{L^\infty(Q^n)} \leq L, \int g dQ^n = 0  \right \}.
\end{equation*}
With this set defined we have the following two results, which are slight generalizations of existing results in \cite{rigollet2022sample} in that they use a general $r$ instead of $r = 1/2$.
\begin{lemma} \label{lem:strong_concave}
    Let $P,Q \in \mathcal{P}(B(0,r))$. Then
    for each $L$, $\Phi_n$ is $\delta$-strongly concave with respect to the norm $\norm{\cdot}_{L^2(P^n) \otimes L^2(Q^n)}$ on $\mathcal{S}_L$ for $\delta = \exp(-[2L+2r^2]/\varepsilon)/\varepsilon$ in the sense that for any $(f,g),(f',g') \in \mathcal{S}_L$ we have with probability 1 
    \begin{equation*}
        \Phi_n(f,g) - \Phi_n(f',g') \geq \langle \nabla \Phi_n(f,g), (f,g) - (f',g') \rangle_{L^2(P^n) \times L^2(Q^n)} + \frac{\delta}{2} \norm{(f,g) - (f',g')}_{L^2(P^n) \times L^2(Q^n)}^2
    \end{equation*}
\end{lemma}
\begin{proof}
    Fix $(f,g),(f',g') \in \mathcal{S}_L$ and define the function $h:[0,1]\rightarrow \mathbb{R}$ by
    \begin{equation*}
        h(t) \triangleq \Phi_n((1-t)f + tf', (1-t)g+tg').
    \end{equation*}
    Then it suffices to show that $h$ satisfies
    \begin{equation*}
        h''(t) \leq -\delta \norm{(f,g) - (f',g')}_{L^2(P^n) \times L^2(Q^n)}^2 \hspace{1cm} \text{ for all } t \in [0,1].
    \end{equation*}
    Fix $t \in [0,1]$. A direct calculation shows
    \begin{align*}
        h''(t) &= -\frac{1}{\varepsilon n^2} \sum_{i,j=1}^n (f(X_i) - f'(X_i) + g(Y_j) - g'(Y_j))^2 \\
        &\hspace{2cm} \times \exp\left ( - \frac{1}{\varepsilon} \left [ t(f(X_i) + g(Y_j) + (1-t)(f'(X_i) + g'(Y_k)) ) - \frac{1}{2}\norm{X_i - Y_j}^2 \right ] \right ).
    \end{align*}
    By the bounds $|f|,|f'|,|g|,|g'| \leq L$ and $\norm{X_i - Y_j} \leq 2r$, we have
    \begin{equation*}
        h''(t) \leq -\frac{1}{\varepsilon n^2} \exp(-[2L + 2r^2]/\varepsilon) \sum_{i,j=1}^n (f(X_i) - f'(X_i) + g(Y_j) - g'(Y_j))^2
    \end{equation*}
    The last sum can be re-factored as
    \begin{align*}
        &\frac{1}{n^2}\sum_{i,j=1}^n (f(X_i) - f'(X_i) + g(Y_j) - g'(Y_j))^2 \\
        &= \norm{(f,g) - (f',g')}^2_{L^2(P^n) \times L^2(Q^n)} + \frac{2}{n^2}\left [\sum_{i=1}^n f(X_i) - f'(X_i) \right ]\left [ \sum_{j=1}^n g(Y_j) - g'(Y_j)\right ] \\
        &= \norm{(f,g) - (f',g')}^2_{L^2(P^n) \times L^2(Q^n)}.
    \end{align*}
    On the second line the latter term is zero because $\int g dQ^n = \int g' dQ^n = 0$ since $g,g' \in \mathcal{S}_L$.
\end{proof}
A direct consequence of strong convexity is the following result known as the Polyak-Łojasiewicz (PL) inequality.
\begin{lemma} \label{lem:pl_ineq}
    Let $P,Q \in \mathcal{P}(B(0,r))$. Let $L > 0$ be such that $(f_n,g_n) \in \mathcal{S}_L$. Then for any $f,g \in \mathcal{S}_L$, 
    \begin{equation*}
        \Phi_n(f_n,g_n) - \Phi_n(f,g) \leq \frac{\varepsilon}{2}e^{[2L+2r^2]/\varepsilon} \norm{\nabla \Phi_n (f,g)}_{L^2(P^n) \times L^2(Q^n)}^2.
    \end{equation*}
\end{lemma}
Combining the two preceding lemmas we have the following result known as the ``error bound"
\begin{lemma} \label{lem:error_bound}
    Let $P,Q \in \mathcal{P}(B(0,r)).$ Let $L > 0$ be such that $(f_n,g_n) \in \mathcal{S}_L$. The for any $f,g \in \mathcal{S}_L$,
    \begin{equation}\label{eqn:GradPhi_nRHS}
        \norm{(f_n,g_n) - (f,g)}_{L^2(P^n) \times L^2(Q^n)}^2 \leq \varepsilon^2e^{4[L+r^2]/\varepsilon} \norm{\nabla \Phi_n(f,g)}^2_{L^2(P^n) \times L^2(Q^n)}.
    \end{equation}
\end{lemma}
\begin{proof}
    Note that since $(f_n,g_n)$ is optimal for $\Phi_{n}$ we have $\nabla \Phi_n(f,g) = 0$. Therefore by Lemmas \ref{lem:strong_concave} and \ref{lem:pl_ineq} we have
    \begin{align*}
        \frac{\exp(-[2L + 2r^2]/\varepsilon)}{2\varepsilon} \norm{(f_n,g_n) - (f,g)}^2_{L^2(P^n) \times L^2(Q^n)} &\leq \Phi_n(f_n,g_n) - \Phi_n(f,g) \\
        &\leq \frac{\varepsilon}{2}e^{[2L+2r^2]/\varepsilon} \norm{\nabla \Phi_n (f,g)}_{L^2(P^n) \times L^2(Q^n)}^2 .
    \end{align*}
    Multiplying the first and last by $2\varepsilon\exp([2L + 2r^2]/\varepsilon)$ gives
    \begin{equation*}
        \norm{(f_n,g_n) - (f,g)}^2_{L^2(P^n) \times L^2(Q^n)} \leq \varepsilon^2e^{4[L+r^2]/\varepsilon} \norm{\nabla \Phi_n (f,g)}_{L^2(P^n) \times L^2(Q^n)}^2.
    \end{equation*}
\end{proof}
We now proceed to an upper bound in expectation of the quantity on the right hand side of (\ref{eqn:GradPhi_nRHS}). This is the first point at which sampling patterns play any role as well as the first time we derive a novel result which is not essentially contained in \cite{rigollet2022sample}.
\begin{lemma} \label{lem:norm_of_grad}
    Suppose that $P_X,P_Y,Q \in \mathcal{P}(B(0,r))$.
    Let $X_1,...X_n \sim P_X, Y_1,...,Y_n \sim P_Y$ and $Z_1,...,Z_{2n} \sim Q$ be jointly independent samples. Let $P_{1/2} = \frac{1}{2}P_X + \frac{1}{2}P_Y$. Let $\Phi_n$ denote the dual objective between $\frac{1}{2}(P_X^n + P_Y^n)$ and $Q^{2n}$ and let $f_*,g_*$ be the optimal entropic potentials between $P_{1/2}$ and $Q$ with $\int g dQ = 0$. Then
    \begin{equation*}
        \mathbb{E} \norm{\nabla \Phi_n(f_*,g_*)}^2_{L^2((P_X^n + P_Y^n)/2) \times L^2(Q^{2n})} \leq \frac{9\exp(8r^2/\varepsilon)}{8n}
    \end{equation*}
    where the expectation is with respect to the samples $X_1,...,X_n,Y_1,...,Y_n,Z_1,...,Z_{2n}$.
\end{lemma}
\begin{proof}
    Using (\ref{eq:norm_of_grad}), taking expectations and applying linearity we have
    \begin{align*}
        &\mathbb{E} \norm{\nabla \Phi_n(f_*,g_*)}^2_{L^2((P_X^n + P_Y^n)/2) \times L^{2}(Q^{2n})} \\
        &= \mathbb{E} \left [ \frac{1}{2n} \sum_{i=1}^n \left ( 1 - \frac{1}{2n} \sum_{j=1}^{2n} p_*(X_i,Z_j) \right )^2 \right ] + \mathbb{E} \left [ \frac{1}{2n} \sum_{i=1}^n \left ( 1 - \frac{1}{2n} \sum_{j=1}^{2n} p_*(Y_i,Z_j) \right )^2 \right ] \\
        & \hspace{1cm}+ \mathbb{E} \left [ \frac{1}{2n} \sum_{j=1}^{2n} \left ( 1 - \frac{1}{2n} \sum_{i=1}^{n} p_*(X_i,Z_j) - \frac{1}{2n} \sum_{i=1}^n p_*(Y_i,Z_j) \right )^2 \right ] 
    \end{align*}
    We will handle the three terms separately. For the first term we have
    \begin{align*}
        \mathbb{E} \left [ \frac{1}{2n} \sum_{i=1}^n \left ( 1 - \frac{1}{2n} \sum_{j=1}^{2n} p_*(X_i,Z_j) \right )^2 \right ] 
        &= \frac{1}{2} \mathbb{E} \left [ \left ( \frac{1}{2n}\sum_{j=1}^{2n} (1 - p_*(X_1,Z_j)) \right )^2 \right ] \\
        &= \frac{1}{2} \frac{1}{4n^2} \sum_{j,k=1}^{2n} \mathbb{E} \left [ (1 - p_*(X_1,Z_j))(1 - p_*(X_1,Z_k))\right ] \\
        &= \frac{1}{8n^2} \sum_{j=1}^{2n} \mathbb{E} \left [ (1 - p_*(X_1,Z_j))^2 \right ] \\
        &= \frac{1}{4n} \text{Var}(p_*(X_1,Z_j)) \\
        &\leq \frac{1}{4n} \frac{\exp(8r^2/\varepsilon)}{4} \\
        &= \frac{\exp(8r^2/\varepsilon)}{16n},
    \end{align*}
    where the first line uses the fact that the $X_i$ are identically distributed, the second is just factoring and linearity of expectation, the third uses the fact that $\mathbb{E}[(1- p_*(X_1,Z_j))(1-p_*(X_1,Z_k))] = 0$ if $j \neq k$ (see below), the fourth uses that $\mathbb{E}p_*(X_1,Z_j) = 1$, and the fifth uses that $p_* \in [0,\exp(4r^2/\varepsilon)]$ and Popoviciu's inequality. The zero-mean formula is verified as follows
    \begin{align*}
        \mathbb{E}[(1- p_*(X_1,Z_j))(1-p_*(X_1,Z_k))] &= \mathbb{E}_{X_1} \left [ \mathbb{E}_{Z_j,Z_k} \left [ (1- p_*(X_1,Z_j))(1-p_*(X_1,Z_k))   \bigg | X_1 \right ] \right ] \\
        &= \mathbb{E}_{X_1} \left [ \mathbb{E}_{Z_j} \left [ 1 - p_*(X_1,Z_j) \bigg | X_1 \right ]\mathbb{E}_{Z_k} \left [ 1 - p_*(X_1,Z_k) \bigg | X_1 \right ]\right ] \\
        &= \mathbb{E}_{X_1} [ (0)(0) ] = 0
    \end{align*}
    where we have used that $1 - p_{*}(X_1,Z_j)$ is conditionally independent of $1 - p_{*}(X_1,Z_k)$ given $X_1$, followed by the marginal constraint on the optimal dual potentials.

    Replacing $X$ with $Y$ in the calculations above immediately gives
    \begin{equation*}
        \mathbb{E} \left [ \frac{1}{2n} \sum_{i=1}^n \left ( 1 - \frac{1}{2n} \sum_{j=1}^{2n} p_*(Y_i,Z_j) \right )^2 \right ]  \leq \frac{\exp(8r^2/\varepsilon)}{16n}.
    \end{equation*}
    This handles the first two terms which must be controlled. We now turn our focus to the last term which is the most challenging to handle.
    \begin{align*}
        & \mathbb{E} \left [ \frac{1}{2n} \sum_{j=1}^{2n} \left ( 1 - \frac{1}{2n} \sum_{i=1}^{n} p_*(X_i,Z_j) - \frac{1}{2n} \sum_{i=1}^n p_*(Y_i,Z_j) \right )^2 \right ] \\
        &= \mathbb{E} \left [ \left ( \frac{1}{2n} \sum_{i=1}^{n} (1 - p_*(X_i,Z_1)) + \frac{1}{2n} \sum_{i=1}^n (1 - p_*(Y_i,Z_1)) \right )^2  \right ] \\
        &= \mathbb{E} \left [ \frac{1}{4n^2}\sum_{i,k=1}^n (1-p(X_i,Z_1))(1-p(X_k,Z_1)) + 2(1-p_*(X_i,Z_1))(1-p_*(Y_k,Z_1)) \right. \\
        & \hspace{2cm} \left . + (1-p_*(Y_i,Z_1))(1-p_*(Y_k,Z_1)) \right ] \\
        &= \frac{1}{4n^2} \sum_{i=1}^n \mathbb{E} \left [ (1-p_*(X_i,Z_1))^2 + 2(1-p(X_i,Z_1))(1-p_*(Y_i,Z_1)) + (1 - p_*(Y_i,Z_i))^2 \right ] \\
        &\leq \frac{1}{4n} 4\exp(8r^2/\varepsilon) = \frac{\exp(8r^2/\varepsilon)}{n}
    \end{align*}
    where the second line is just linearity, iid assumptions, and refactoring. The third expanding the squares. The fourth uses a fact shown below, and the fifth uses the uniform bounds on $p_*$ from Lemma \ref{lem:bounded_potentials} which in turn implies an upper bound on $1 - p_*$. The fact that we must show is that for $i \neq k$
    \begin{equation*}
        \mathbb{E} \left [ (1-p_*(X_i,Z_1))(1-p_*(X_k,Z_1)) + 2 (1-p_*(X_i,Z_1))(1-p_*(Y_i,Z_1)) + (1-p_*(Y_i,Z_1))(1-p_*(Y_k,Z_1)) \right ] = 0
    \end{equation*}
    Note that by the marginal constraint we know that conditioned on $Z$
    \begin{equation*}
        1 = \mathbb{E}_{W \sim P_{1/2}}\left [ p_*(W,Z) \right ] = \mathbb{E}_{X,Y} [(1/2)p_*(X,Z) + (1/2)p_*(Y,Z)] = \frac{1}{2}\mathbb{E}_{X}[p_*(X,Z)] + \frac{1}{2}\mathbb{E}_{Y}[p_*(Y,Z)].
    \end{equation*}
    Rearranging the first and last equalities shows
    \begin{equation*}
        \mathbb{E}_X[1 - p_*(X,Z)] = \mathbb{E}_Y[p_*(Y,Z) - 1] = -\mathbb{E}_Y[1 - p_*(Y,Z)]
    \end{equation*}
    which is the crucial identity that we require. Using this we have conditioned on $Z_1$ that
    \begin{align*}
        &\mathbb{E}[(1-p_*(X_i,Z_1))(1-p_*(X_k,Z_1)) + (1-p_*(X_i,Z_1))(1-p_*(Y_k,Z_1))] \\
        &= \mathbb{E}[(1-p_*(X_i,Z_1))]\mathbb{E}[(1-p_*(X_k,Z_1))] + \mathbb{E}[(1-p_*(X_i,Z_1))]\mathbb{E}[(1-p_*(Y_k,Z_1))] \\
        &= \mathbb{E}[(1-p_*(X_i,Z_1))]\mathbb{E}[(1-p_*(X_k,Z_1))] + \mathbb{E}[(1-p_*(X_i,Z_1))] \left ( -\mathbb{E}[(1-p_*(X_k,Z_1))]\right ) \\
        &= 0
    \end{align*}
    Similarly
    \begin{align*}
        &\mathbb{E}[(1-p_*(Y_i,Z_1))(1-p_*(Y_k,Z_1)) + (1-p_*(X_i,Z_1))(1-p_*(Y_k,Z_1))] \\
        &= \mathbb{E}[(1-p_*(Y_i,Z_1))]\mathbb{E}[(1-p_*(Y_k,Z_1))] + \mathbb{E}[(1-p_*(X_i,Z_1))]\mathbb{E}[(1-p_*(Y_k,Z_1))] \\
        &= \mathbb{E}[(1-p_*(Y_i,Z_1))]\mathbb{E}[(1-p_*(Y_k,Z_1))] + \left (-\mathbb{E}[(1-p_*(Y_i,Z_1))] \right )\mathbb{E}[(1-p_*(Y_k,Z_1))] \\
        &= 0
    \end{align*}
    Adding these expressions proves the required result.

    Tracking back the bounds above we have shown 
    \begin{equation*}
        \mathbb{E} \norm{\nabla \Phi_n(f_*,g_*)}^2_{L^2((P_X^n + P_Y^n)/2) \times L^{2}(Q^{2n})} \leq \frac{\exp(8r^2/\varepsilon)}{16n} + \frac{\exp(8r^2/\varepsilon)}{16n} + \frac{\exp(8r^2/\varepsilon)}{n} = \frac{9\exp(8r^2/\varepsilon)}{8n}.
    \end{equation*}
\end{proof}
By combining Lemmas \ref{lem:bounded_potentials}, \ref{lem:error_bound}, and \ref{lem:norm_of_grad} we achieve the following bound in the deviation of the potentials. This bound compares the estimated potentials to the true potentials once they have been appropriately shifted to account for the fact that the optimal potentials are only unique up to an additive constant.
\begin{lemma} \label{lem:fngn_bound}
    Consider the setting of Lemma \ref{lem:norm_of_grad}. Then $\overline{f_*}, \overline{g_*}$ satisfy
    \begin{equation*}
        \mathbb{E} \norm{(f_n,g_n) - (\overline{f}_* ,\overline{g}_*)}^2_{L^2((P_X^n + P_Y^n)/2) \times L^2(Q^{2n})} \leq \frac{9\varepsilon^2}{8n}\exp(28r^2/\varepsilon).
    \end{equation*}
\end{lemma}
\begin{proof}
    Let $L$ be such that $f_n,g_n,\overline{f}_*,\overline{g}_*$ are with probability 1 contained in $\mathcal{S}_L$. Taking expectations in Lemma \ref{lem:error_bound} and then applying Lemma \ref{lem:norm_of_grad} we have
    \begin{align*}
        \mathbb{E} \norm{(f_n,g_n) - (\overline{f}_* ,\overline{g}_*)}^2_{L^2((P_X^n + P_Y^n)/2) \times L^{2}(Q^{2n})} &\leq \varepsilon^2e^{4[L+r^2]/\varepsilon} \mathbb{E}\norm{\nabla \Phi_n(\overline{f}_*,\overline{g}_*)}^2_{L^2((P_X^n + P_Y^n)/2)\times L^{2}(Q^{2n})} \\
        &= \varepsilon^2e^{4[L+r^2]/\varepsilon} \mathbb{E}\norm{\nabla \Phi_n(f_*,g_*)}^2_{L^2((P_X^n + P_Y^n)/2)\times L^{2}(Q^{2n})} \\
        &\leq \varepsilon^2e^{4[L+r^2]/\varepsilon} \frac{9\exp(8r^2/\varepsilon)}{8n}.
    \end{align*}
    All that remains is to find an $L$ sufficiently large. Note that by Lemma \ref{lem:bounded_potentials} we have 
    \begin{align*}
        |\overline{f}_*(X_i)| &= \left |f_*(X_i) + \frac{1}{2n} \sum_{j=1}^{2n} g_n(Z_j) \right | \\
        &\leq |f_*(X_i)| + \frac{1}{2n} \sum_{j=1}^{2n} |g_n(Z_j)| \\
        &\leq 2r^2 + 2r^2 = 4r^2,
    \end{align*}
    and by a similar calculation
    \begin{equation*}
        |\overline{g}_*(Z_j)| \leq 4r^2.
    \end{equation*}
    From this it follows that $f_n,g_n,\overline{f}_*,\overline{g}_* \in \mathcal{S}_{4r^2}$. Using this setting of $L$ in the bound above we have
    \begin{align*}
        \mathbb{E} \norm{(f_n,g_n) - (\overline{f}_* ,\overline{g}_*)}^2_{L^2((P_X^n + P_Y^n)/2) \times L^{2}(Q^{2n})} &\leq \varepsilon^2e^{4[4r^2+r^2]/\varepsilon} \frac{9\exp(8r^2/\varepsilon)}{8n} \\
        &= \frac{9\varepsilon^2}{8n}\exp(28r^2/\varepsilon).
    \end{align*}
\end{proof}
By observing that the relative densities $p_*$ and $p_n$ are determined by the dual potentials $(f_*,g_*)$ and $(f_n,g_n)$ one can immediately convert the result above into a bound on the relative densities as follows.
\begin{lemma}
    Consider the setting of Lemma \ref{lem:norm_of_grad}. Then the relative density $p_n$ satisfies
    \begin{equation*}
        \mathbb{E}\norm{p_n - p_*}^2_{L^2((P_X^n + P_Y^n)/2 \otimes Q^{2n})} \leq \frac{9\varepsilon^2}{4n} \exp(44r^2/\varepsilon).
    \end{equation*}
\end{lemma}
\begin{proof}
    To start note that for every $(w,z) \in (\{X_1,...,X_n\} \cup \{Y_1,...,Y_n\}) \times \{Z_1,...,Z_{2n}\}$ we have
    \begin{align*}
        |p_n(w,z) - p_*(w,z)| &= \left | \exp \left ( \frac{1}{\varepsilon} \left [ f_n(w) + g_n(z) - \frac{1}{2} \norm{w-z}^2 \right ] \right ) - \exp \left ( \frac{1}{\varepsilon} \left [ f_*(w) + g_*(z) - \frac{1}{2}\norm{w-z}^2 \right ] \right ) \right | \\
        &= \left | \exp \left ( \frac{1}{\varepsilon} \left [ f_n(w) + g_n(z) - \frac{1}{2} \norm{w-z}^2 \right ] \right ) - \exp \left ( \frac{1}{\varepsilon} \left [ \overline{f}_*(w) + \overline{g}_*(z) - \frac{1}{2}\norm{w-z}^2 \right ] \right ) \right | \\
        &\leq e^{\frac{1}{\varepsilon}8r^2} \left | f_n(w) - \overline{f}_*(w) + g_n(z) - \overline{g}_*(z) \right | \\
        &\leq e^{\frac{1}{\varepsilon}8r^2} \left | f_n(w) - \overline{f}_*(w) \right | + e^{\frac{1}{\varepsilon}8r^2} \left | g_n(z) - \overline{g}_*(z) \right | 
    \end{align*} 
    where we have used the fact that $f_n,g_n,f_*,g_* \in S_{4r^2}$ (see above) followed by the fact that $e^{t}$ is $e^{C}$-Lipschitz over $(-\infty,C]$. 

    From here we can compute
    \begin{align*}
        &\norm{p_n - p_*}^2_{L^2((P_X^n + P_Y^n)/2 \otimes Q^{2n})} \\
        &= \frac{1}{4n^2}\sum_{i=1}^n\sum_{j=1}^{2n} |p_n(X_i,Z_j) - p_*(X_i,Z_j)|^2 + |p_n(Y_i,Z_j) - p_*(Y_i,Z_j)|^2 \\
        &\leq \frac{1}{4n^2} \sum_{i=1}^n \sum_{j=1}^{2n} 2e^{16r^2/\varepsilon} |f_n(X_i) - \overline{f}_*(X_i)| + 2e^{16r^2/\varepsilon} |f_n(Y_i) - \overline{f}_*(Y_i)| + 4e^{16r^2/\varepsilon} |g_n(Z_j) - \overline{g}_*(Z_j)| \\
        &= 2e^{16r^2/\varepsilon} \norm{f_n - \overline{f}_*}^2_{L^2((P^n_X + P^n_Y)/2)} + 2e^{16r^2/\varepsilon} \norm{g_n - \overline{g}_*}^2_{L^2(Q^{2n})} \\
        &= 2e^{16r^2/\varepsilon} \norm{(f_n,g_n) - (\overline{f}_*,\overline{g}_*)}^2_{L^2((P^n_X + P^n_Y)/2) \times L^2(Q^{2n})}.
    \end{align*}
    Taking expectations on the first and last and applying Lemma \ref{lem:fngn_bound} we have
    \begin{align*}
        \mathbb{E}\norm{p_n - p_*}^2_{L^2((P_X^n + P_Y^n)/2 \otimes Q^{2n})} &\leq 2e^{16r^2/\varepsilon}\mathbb{E}\norm{(f_n,g_n) - (\overline{f}_*,\overline{g}_*)}^2_{L^2((P^n_X + P^n_Y)/2) \times L^2(Q^{2n})} \\
        &\leq \frac{9\varepsilon^2}{4n} \exp(44r^2/\varepsilon).
    \end{align*}
\end{proof}
Now we can proceed to bounding the deviation of the entropic map on the samples.
\begin{lemma} \label{lem:map_estimate}
    Consider the setting of Lemma \ref{lem:fngn_bound}. Then the entropic map $\Tepsn$ satisfies
    \begin{equation*}
        \mathbb{E} \norm{\Tepsn - \Teps}^2_{L^2((P^n_X + P^n_Y)/2)} \leq\frac{9r^2(1+\varepsilon^2)}{2n}\exp(44r^2/\varepsilon) 
    \end{equation*}
\end{lemma}
\begin{proof}
    For each sample $w \in \{X_1,...,X_n\} \cup \{Y_1,...,Y_n\}$ we have the bound
    \begin{align*}
        \norm{\Tepsn(w) - \Teps(w)}^2 &= \norm{\frac{1}{2n}\sum_{j=1}^{2n}p_n(w,Z_j)Z_j - \int p_*(w,z)zdQ(z)}^2 \\
        &\leq 2 \norm{\frac{1}{2n} \sum_{j=1}^{2n} (p_n(w,Z_j) - p_*(w,Z_j))Z_j}^2 + 2 \norm{\frac{1}{2n} \sum_{j=1}^{2n} p_*(w,Z_j)Z_j - \int p_*(w,z)zdQ(z)}^2
    \end{align*}
    We will handle these two terms separately. For the first we have by Jensen's inequality followed by the boundedness of $Q$
    \begin{align*}
        \norm{\frac{1}{2n} \sum_{j=1}^{2n} (p_n(w,Z_j) - p_*(w,Z_j))Z_j}^2  
        &\leq \frac{1}{2n} \sum_{j=1}^{2n} \norm{p_n(x,Z_j) - p_*(w,Z_j)Z_j}^2 \\
        &\leq \frac{1}{2n} \sum_{j=1}^{2n} r^2 \left ( p_n(w,Z_j) - p_*(w,Z_j) \right )^2.
    \end{align*}
    For the second term we can expand the square and take expectation over the $Z_j$ to obtain
    \begin{align*}
        &\mathbb{E}\norm{\frac{1}{2n} \sum_{j=1}^{2n} p_*(w,Z_j)Z_j - \int p_*(w,z)zdQ(z)}^2 \\
        &= \frac{1}{4n^2} \sum_{j,k=1}^{2n} \mathbb{E} \left \langle  p_*(w,Z_j)Z_j - \int p_*(w,z)zdQ(z),  p_*(w,Z_k)Z_k - \int p_*(w,z)zdQ(z) \right \rangle \\
        &= \frac{1}{2n} \mathbb{E} \norm{p_*(w,Z_1)Z_1 - \int p_*(w,z)zdQ(z) }^2
    \end{align*}
    where we have used that for fixed $x$ that $p_*(x,Z_j)Z_j - \int p_*(x,z)zdQ(z)$ and $p_*(x,Z_k)Z_k - \int p_*(x,z)zdQ(z)$ are zero-mean and independent for all $j \neq k$, which implies that the cross terms cancel.
    We can further bound this by using Lemma \ref{lem:bounded_potentials}
    \begin{align*}
        \mathbb{E} \norm{p_*(x,Z_1)Z_1 - \int p_*(x,z)zdQ(z) }^2 &\leq \norm{p_*}_\infty \mathbb{E}\norm{Z_1 - \int p_*(x,z)zdQ(z)}^2 \\
        &\leq \norm{p_*}_{\infty} (2r)^2 \leq 4r^2\exp(8r^2/\varepsilon),
    \end{align*}
    Combining the inequalities derived we have
    \begin{align*}
        &\mathbb{E} \norm{\Tepsn - \Teps}^2_{L^2((P^n_X + P^n_Y)/2)} \\
        &= \mathbb{E} \frac{1}{2n} \sum_{i=1}^{n} \norm{\Tepsn(X_i) - \Teps(X_i)}^2 + \norm{\Tepsn(Y_i) - \Teps(Y)}^2 \\
        &\leq \mathbb{E} \frac{1}{n} \sum_{i=1}^n \left [ \frac{r^2}{2n} \sum_{j=1}^{2n} (p_n(X_i,Z_j) - p_*(X_i,Z_j))^2 + (p_n(Y_i,Z_j) - p_*(Y_i,Z_j))^2 \right ] \\
        &+ \mathbb{E}\frac{1}{n} \sum_{i=1}^n \left [ \frac{1}{2n} 4r^2\exp(8r^2/\varepsilon) + \frac{1}{2n} 4r^2\exp(8r^2/\varepsilon) \right ] \\
        &= 2r^2 \mathbb{E} \norm{p_n - p_*}^2_{L^2((P^n_X + P^n_Y)/2 \otimes Q^{2n})} + \frac{4r^2}{n}\exp(8r^2/\varepsilon) \\
        &\leq \frac{9r^2\varepsilon^2}{2n}\exp(44r^2/\varepsilon) + \frac{4r^2}{n}\exp(8r^2/\varepsilon)\\ 
        &\leq \frac{9r^2(1+\varepsilon^2)}{2n}\exp(44r^2/\varepsilon).
    \end{align*}
\end{proof}
\section{Additional Technical Results}
\begin{lemma} \label{lem:h_sum_bound} Let $P,Q$ be any probability measures and let $X,X',X_1,...X_n \sim P$ and $Y,Y',Y_1,...,Y_n \sim Q$ be jointly independent. Let $h:\mathbb{R}^d \times \mathbb{R}^d \rightarrow \mathbb{R}$ such that $h(x,x) = 0$ for every $x \in \mathbb{R}^d$. Then
    \begin{equation*}
        \mathbb{E} \left | \frac{1}{n^2} \sum_{i,j=1}^n [2h(X_i,Y_j) - h(X_i,X_j) - h(Y_i,Y_j)] - \mathbb{E}[2h(X,Y) - h(X,X') - h(Y,Y')] \right | \leq 8\norm{h}_\infty \sqrt{\frac{\pi}{n}}
    \end{equation*}
\end{lemma}
\begin{proof}
    The proof leverages the bounded differences inequality \cite{mcdiarmid1989method} for the function 
    \begin{equation*}
        H_{h}(x_1,...,x_n,y_1,...,y_n) = \sum_{i,j=1}^n \frac{2}{n^2}h(x_i,y_j) - \frac{1}{n^2}\sum_{i=1}^n [h(x_i,x_j) + h(y_i,y_j)].
    \end{equation*}
    This function satisfies the bounded differences property for each variable $x_i$:
    \begin{align*}
        &|H_{h}(x_1,x_2,...,x_n,y_1,...,y_n) - H_{h}(x_1',x_2,...,x_n,y_1,...,y_n)| \\
        &= \left | \frac{2}{n^2} \sum_{i=1}^n [h(x_1,y_i) - h(x_1',y_i)] - \frac{2}{n^2} \sum_{i=2}^n[h(x_1,x_i) - h(x_1',x_i)] \right | \\
        &\leq \frac{2}{n^2} \sum_{i=1}^n [|h(x_1,y_i)| + |h(x_1',y_i)|] + \frac{2}{n^2} \sum_{i=2}^n [|h(x_1,x_i)| + |h(x_1',x_i)|] \\
        &\leq \frac{2}{n^2} n 2\norm{h}_\infty + \frac{2}{n^2} (n-1) 2\norm{h}_\infty \leq \frac{8}{n}\norm{h}_\infty
    \end{align*}
    where we have used that $h(x_1,x_1) = h(x_1',x_1') = 0$ by the assumptions on $h$. An analogous computation holds for every other $x_i$ and $y_i$. Next note that
    \begin{align*}
         &\mathbb{E}H_{h}(X_1,...,X_n,Y_1,...,Y_n) \\
         &= \mathbb{E} \left [ \frac{2}{n^2} \sum_{i,j=1}^n h(X_i,Y_j) - \frac{1}{n^2} \sum_{i,j=1}^n h(X_i,X_j) - \frac{1}{n^2} \sum_{i,j=1}^n h(Y_i,Y_j)\right ] \\
         &= \mathbb{E} \left [ 2h(X,Y) - h(X,X') - h(Y,Y') \right ] 
    \end{align*}
    Therefore by the bounded differences inequality we have
    \begin{align*}
        &\mathbb{P} \left \{ \left | \frac{1}{n^2} \sum_{i,j=1}^n [2h(X_i,Y_j) - h(X_i,X_j) - h(Y_i,Y_j)] - \mathbb{E}[2h(X,Y) - h(X,X') - h(Y,Y')] \right | > t\right \} \\
        &= \mathbb{P} \left \{ \left | H_{h}(X_1,...,X_n,Y_1,...,Y_n) - \mathbb{E}H_{h}(X_1,...,X_n,Y_1,...,Y_n) \right | > t \right \} \\
        &\leq 2\exp \left ( \frac{-2t^2}{2n(8\norm{h}_\infty/n)^2}\right ) \\
        &= 2\exp\left ( \frac{-nt^2}{64\norm{h}_\infty^2} \right )
    \end{align*}
    Now using the tail bound form of expectation we have
    \begin{align*}
        &\mathbb{E}\left | \frac{1}{n^2} \sum_{i,j=1}^n [2h(X_i,Y_j) - h(X_i,X_j) - h(Y_i,Y_j)] - \mathbb{E}[2h(X,Y) - h(X,X') - h(Y,Y')] \right | \\
        &= \int_0^t \left \{ \left | \frac{1}{n^2} \sum_{i,j=1}^n [2h(X_i,Y_j) - h(X_i,X_j) - h(Y_i,Y_j)] - \mathbb{E}[2h(X,Y) - h(X,X') - h(Y,Y')] \right | > t\right \} dt \\
        &\leq 2 \int_0^t \exp \left ( \frac{-nt^2}{64\norm{h}_\infty^2} \right ) dt = 8\norm{h}_\infty \sqrt{\frac{\pi}{n}}
    \end{align*}
\end{proof}

\begin{theorem} [Theorem 1, \cite{cordero2019regularity}] \label{thm:erausquin_figali} Let $\Omega_X, \Omega_Y$ be two open subsets of $\mathbb{R}^d$. Let $P$ and $Q$ be two measures with densities $f,g$ on $\Omega_X, \Omega_Y$ respectively. Assume that:
\begin{itemize}
    \item [(A1)] $\Omega_Y$ is bounded and convex.
    \item [(A2)] $| \partial \Omega_X| = 0$, i.e. the Lebesgue measure of the boundary of $\Omega_X$ is zero.
    \item [(A3)] $f, \ 1/f \in L^{\infty}(\Omega_X \cap B_{R})$, $g, \ 1/g \in L^{\infty}(\Omega_Y \cap B_{R})$ for all $R> 0$, i.e. the densities restricted to the intersection of their respective domains with Euclidean balls of radius $R$ are bounded above and below.
    
\end{itemize}
Then the OT map $T: \Omega_X \rightarrow \Omega_Y$ transporting $P$ to $Q$ is continuous and $T(\Omega_X)$ is an open subset of $\Omega_Y$ of full measure inside $\Omega_Y$.
\end{theorem}
\end{document}